\documentclass[twocolumn,amsthm]{autart}    
\usepackage[utf8]{inputenc}

\usepackage{comment}
\usepackage{tikz}
\usetikzlibrary{calc,patterns,angles,quotes}

\usepackage{amsmath} 
\usepackage{bbm}

\DeclareMathOperator{\atantwo}{atan2}
\usepackage{booktabs}

\usepackage{amssymb}
\usepackage{bm}
\usepackage{mathtools}
\usepackage{amssymb}

\theoremstyle{definition}
\newtheorem{definition}{Definition}
\newtheorem{theorem}{Theorem}

\newtheorem{assumption}{Assumption}
\newtheorem{corollary}{Corollary}
\newtheorem{problem}{Problem}

\newtheorem{remark}{Remark}

\newtheorem{example}{Example}

\usepackage{hyperref}
\hypersetup{
    colorlinks=true,
    linkcolor=black,
    filecolor=black,      
    urlcolor=blue,
    citecolor=blue,
    }
\usepackage[round]{natbib}
\usepackage{wrapfig}

\newcommand{\blue}[1]{\textcolor{black}{#1}}

\usepackage{todonotes}

\begin{document}

\begin{frontmatter}

\title{Optimal Control of Differentially Flat Systems is Surprisingly Easy\thanksref{footnoteinfo}}

\thanks[footnoteinfo]{This research was supported by NSF under Grants CNS-2149520 and CMMI-2219761.}

\author[ODU]{Logan E. Beaver}\ead{lbeaver@odu.edu},
\author[CORNELL]{Andreas A. Malikopoulos}\ead{amaliko@cornell.edu}

\address[ODU]{Department of Mechanical and Aerospace Engineering, Old Dominion University, Norfolk, VA 23529}
\address[CORNELL]{School of Civil and Environmental Engineering, Cornell University, Ithaca, NY 14853}  


\begin{abstract}                          
As we move to increasingly complex cyber-physical systems (CPS), new approaches are needed to plan efficient state trajectories in real-time.
In this paper, we propose an approach to significantly reduce the complexity of solving optimal control problems for a class of CPS with nonlinear dynamics.
We exploit the property of differential flatness to simplify the Euler-Lagrange equations that arise during optimization, and this simplification eliminates the numerical instabilities that plague optimal control in general.
We also present an explicit differential equation that describes the evolution of the optimal state trajectory, and we extend our results to consider both the unconstrained and constrained cases.
Furthermore, we demonstrate the performance of our approach by generating the optimal trajectory for a \blue{planar manipulator with two revolute joints}.
We show in simulation that our approach is able to generate the constrained optimal trajectory in \blue{$4.5$ ms while respecting workspace constraints and switching between a `left' and `right' bend in the elbow joint.}
\end{abstract}

\end{frontmatter}

\section{Introduction}

There is an increasing demand to extend the boundaries of autonomy in cyber-physical systems (CPS) using experimental testbeds (see: \cite{Rubenstein2012,jang2019simulation,Beaver2020DemonstrationCity,chalaki2021CSM}) and outdoor experiments (see:  \cite{Vasarhelyi2018OptimizedEnvironments,mahbub2020sae-2,chalaki2021CSM}).
As CPS achieve higher autonomy levels, they will be forced into complicated interactions with other agents and the surrounding environment \citep{Malikopoulos2020,Beaver2020AnFlockingb,Oh2017}.
These autonomous agents must be able to react quickly to their environment and re-plan efficient trajectories.
To this end, we propose a new method to simplify real-time optimal trajectory planning by exploiting differential flatness.

A system is differentially flat if there exist a set of endogenous \emph{flat} variables, also called \emph{outputs}, such that the original state and control variables can be written as an explicit function of the flat variables and a finite number of their derivatives.
This yields an equivalent flat system that is completely described by integrator dynamics.
It is significantly easier to generate control trajectories in the flat space, wherein the trajectories can be exactly mapped back to the original coordinate system.
Differentially flat systems have garnered significant interest since their introduction by \cite{Fliess1995FlatnessExamples}, and it has been shown that generating trajectories in the flat space can reduce computational time by at least an order of magnitude (e.g., see:  \cite{Petit2001InversionOptimization}).
Differentially flat systems are closely related to feedback linearizable systems \citep{Levine2007OnLinearizability}; however, the standard control techniques for flat systems are distinct from feedback linearization.

The overwhelming majority of research on trajectory generation with differential flatness uses collocation techniques, i.e., finding optimal parameters for a set of basis functions in the flat space.
Under this approach, a designer selects an appropriate basis function for their application, e.g., polynomial splines in \cite{Mellinger2011MinimumQuadrotors,Sreenath2013TrajectorySystem}, Bezier curves in \cite{Milam2003Real-TimeSystems}, Fourier transforms in \cite{Ogunbodede2020OptimalSystems}, or piece-wise constant functions in \cite{Kolar2017Time-optimalCrane}.
The parameters of these basis functions are optimally determined to yield the optimal trajectory for the selected basis.
A rigorous overview of this approach is given in the recent textbook by \cite{Sira-Ramirez2018DifferentiallySystems}.

In contrast, we propose an indirect approach that seeks a solution by solving a set of optimality conditions.

Our approach is similar to contemporary methods, such as NOSNOC (see \cite{NOSNOC}) and the Method of Evolving Junctions (MEJ); see \cite{Li2017MethodObstacles}.
Each of these algorithms explicitly resolves the junctions that arise in the optimal control problem.
NOSNOC was developed to solve systems with switched dynamics, and explicitly includes the switching point in its discretization.
Similarly, MEJ has been used for optimal navigation in discrete flow fields (see \cite{Zhai2022MethodFields}), where the boundary between different flow regions are explicitly resolved.
Similarly, our approach generates a collection of optimal trajectory segments between discrete junctions.
However, we determine the optimal junctions using standard root-finding algorithms, whereas NOSNOC discretizes the entire space, and MEJ uses a stochastic global search method.

We also note that the MEJ has primarily been applied to linear systems with quadratic objective functions, e.g., see \cite{Li2017MethodObstacles}.
Furthermore, our approach gives an equation that describes the system's trajectory between junctions, whereas MEJ and NOSNOC give no such construction.

There are also weaker and more general analytical results for the so-called maximal inversion approach by \cite{Chaplais2007InversionCases,Chaplais2008InversionSystems}, which proves that the optimality conditions for a feedback linearizable system can be separated into two parts\textemdash one describing the optimal state trajectory, and the other describing the optimal costate trajectory.
This separation result is significant, as the general optimality conditions couple the evolution of the states and costates, which leads to significant numerical instabilities (see: \cite{Bryson1996Optimal1985}).
While \cite{Chaplais2008InversionSystems} proved that the optimality conditions are  separable, in this paper, we provide the analytical form of the ordinary differential equation that explicitly describes their evolution.
Furthermore, while \cite{Chaplais2008InversionSystems} considers control-affine nonlinear systems, our proposed approach does not require affinity in the control variables.
More recent work following this approach employs saturation functions to handle trajectory constraints, e.g., \cite{Graichen2010HandlingExtension}, whereas our approach explicitly generates constrained optimal trajectories.
Finally, we also derive the optimal boundary conditions in the flat space, which, to the best of our knowledge, has not been addressed in the literature to date.
The contributions of this paper are:
\begin{itemize}
    \item We present a set of ordinary differential equations that describe the evolution of the costates as explicit functions of the state and control variables (Theorem \ref{thm:costates}).
    \item We derive optimality conditions that are independent of the costates (Theorem \ref{thm:main}). This independence property holds for interior-point (Section \ref{sec:interior-const}) and path (Section \ref{sec:trajConst}) constraints.
    \item We derive equivalent boundary conditions for the state and control variables when an initial or final state is left free or when the final time is unknown (Section \ref{sec:boundaries}).
\end{itemize}

The remainder of the article is organized as follows.
In Section \ref{sec:pf}, we provide the modeling framework and enumerate our assumptions before presenting our main theoretical results in Section \ref{sec:main}.
In Section \ref{sec:case-study}, we provide an illustrative example of controlling a nonlinear \blue{planar manipulator, and relate our differential flatness transformations to the forward and inverse kinematics}.
Finally, we draw concluding remarks and present directions of future work in Section \ref{sec:conclusion}.

\section{Problem Formulation} \label{sec:pf}
Consider the nonlinear dynamical system,
\begin{equation} \label{eq:dynamics}
    \dot{\bm{x}}(t) = \bm{f}\big( \bm{x}(t), \bm{u}(t)\big),
\end{equation}
where $\bm{x}(t)\in\mathcal{X}\subset\mathbb{R}^n$ and $\bm{u}(t)\in\mathcal{U}\subset\mathbb{R}^m$, $n \geq m$, are the state and control vectors, respectively, $\bm{f}$ is a smooth vector field, and $t\in\mathbb{R}$ is time.
The system is \emph{differentially flat} if the following definition holds.

\begin{definition}[Adapted from \cite{Rigatos2015DifferentialControl}]\label{def:flat}
A system described by \eqref{eq:dynamics} is said to be differentially flat if there exists a vector of \emph{outputs} $\bm{y}(t) = (y_1(t), \dots, y_m(t)),$ such that:

\begin{enumerate}
\item There exists a smooth function $\sigma$ that maps $\bm{x}(t)$, $\bm{u}(t)$, and a finite number of its derivatives to $\bm{y}$, i.e.,
\begin{align}
    \bm{y}(t) &= \sigma\big(\bm{x}(t), \bm{u}(t), \dot{\bm{u}}(t), \dots, \bm{u}^{(p)}(t)\big), \label{eq:state2flat}
\end{align}
for some $p\in\mathbb{N}$.

\item The variables $\bm{x}(t)$ and $\bm{u}(t)$ can be expressed as smooth functions of $\bm{y}(t)$ and a finite number of its time derivatives, i.e.,
\begin{align}
    \bm{x}(t) &= \gamma_0\big(\bm{y}(t), \dot{\bm{y}}(t), \dots, \bm{y}^{(q)}(t)\big), \label{eq:flat2state}\\
    \bm{u}(t) &= \gamma_1\big(\bm{y}(t), \dot{\bm{y}}(t), \dots, \bm{y}^{(q)}(t) \big), \label{eq:flat2control}
\end{align}
for some $q \in\mathbb{N}$.

\item The vectors $\bm{y}(t), i = 1, \dots, m$ and their time derivatives are differentially independent, i.e., there exists no differential relation satisfying $\eta(\bm{y}, \dot{\bm{y}}, \dots) = 0$ .

\end{enumerate}
Then the variables $y_i(t)$, $i=1, 2, \dots, m$ are the \emph{outputs} of the differentially flat system.
\end{definition}

Definition \ref{def:flat} implies a smooth bijective mappings $\sigma$, $\gamma_0,$ and $\gamma_1$ between the original space, $\mathcal{X}\times\mathcal{U}\times\mathcal{U}^{(1)}\times\dots$, and a flat space $\mathcal{Y}\times\mathcal{Y}^{(1)}\times\dots$.
Furthermore, since this mapping uses only the original state variables and their derivatives, this is said to be an \emph{endogenous} transformation.

For a comprehensive discussion on differential flatness and the topological properties of flat spaces see \cite{Fliess1999ASystems}.
Next, we impose our working assumptions for the analysis of differentially flat systems that satisfy Definition \ref{def:flat}.

\begin{assumption} \label{smp:existence}
The trajectory of the system is contained in an open set where the functions \eqref{eq:state2flat}--\eqref{eq:flat2control} are well-defined.
\end{assumption}

\begin{assumption} \label{smp:bounded}
The control actions in the original and flat spaces are upper and lower bounded.
\end{assumption}

Assumption \ref{smp:existence} is a standard assumption in the literature (see: \cite{VanNieuwstadt1994DifferentialEquivalence}).
It can be relaxed by constraining the trajectory to remain within a subset where \eqref{eq:state2flat}--\eqref{eq:flat2control} are well-defined, and several relaxations of this assumption are discussed in \cite{Milam2003Real-TimeSystems}.

Assumption \ref{smp:bounded} is common in optimal control (see: \cite{Bryson1975AppliedControl}), particularly for physical systems where actuators are ultimately bounded by their physical strength or energy consumption.
This assumption can be relaxed by allowing the control input to take the form of a Dirac delta function, which introduces additional complexity that requires nonsmooth analysis.

\blue{
We note that, for mechanical systems, Assumption \ref{smp:existence} has been proven to hold for a broad class of practical problems.
For example, in the case of robot manipulators the diffeomorphism \eqref{eq:state2flat} is exactly the forward kinematics, and the inverse transformations \eqref{eq:flat2state} and \eqref{eq:flat2control} are exactly the inverse kinematics and inverse dynamics.
While providing an algorithm to determine the inverse kinematics in general is an open problem, the transformations have been derived and tabulated for many systems (see \cite{spong2020robot} for some examples).
We also demonstrate in our example that singularities in the transformations are equivalent to the unconstrained switching points of \cite{Bryson1975AppliedControl}; we treat these as interior point constraints in our case study.
Furthermore, when discontinuities of the first kind appear in these transformations, they can easily be handled by piecing the left and right limits using continuity in the state--which is implied by the differentially flat dynamics and bounded control in Assumption \ref{smp:bounded} \citep{Bryson1975AppliedControl}.
This further motivates our approach, which is robust to these kinds of discontinuities and singularities.
}

Next, as an illustrative example of our approach, we introduce a ``running" example that we will refer back to throughout the manuscript: a unicycle operating in $\mathbbm{R}^2$.
\begin{example}\label{exa:unicycle}
Let $\bm{x}(t) = [p_x(t), p_y(t), \theta(t)]^T$ be the state of a unicycle in the $\mathbb{R}^2$ plane, where $p_x(t)$ and $p_y(t)$ denote the position, and $\theta(t)$ denotes the heading angle. Let $\bm{u}(t) = [u_1(t), u_2(t)]^T$ be the vector of control actions, where $u_1(t)$ and $u_2(t)$ denote the forward and angular velocity, respectively.
Then, the dynamics are given by,
\begin{align}
    \dot{\bm{x}}(t) = 
    \begin{bmatrix}
        u_1(t) \cos\big(\theta(t)\big) \\
        u_1(t) \sin\big(\theta(t)\big) \\
        u_2(t)
    \end{bmatrix}.
\end{align}
This system admits $m=2$ differentially flat base states, $\bm{y}(t) = [y_1(t), y_2(t)]^T = [p_x(t), p_y(t)]^T$ (see \cite{Sira-Ramirez2018DifferentiallySystems}).
The transformations \eqref{eq:flat2state} and \eqref{eq:flat2control} between the flat and original variables are
\begin{align}
    \begin{bmatrix}
        p_x(t) \\ p_y(t) \\ \theta(t)
    \end{bmatrix}
    &=
    \begin{bmatrix}
        y_1(t) \\ y_2(t) \\ \atantwo(\dot{y}_2, \dot{y}_1)
    \end{bmatrix}, \label{eq:gUnicycle}
    \\
    \begin{bmatrix}
        u_1(t) \\ u_2(t)
    \end{bmatrix}
    &= 
    \begin{bmatrix}
        \sqrt{\dot{y}_1(t)^2 + \dot{y}_2(t)^2} \\
        \frac{\ddot{y}_2\dot{y}_1 - \dot{y}_2\ddot{y}_1}{\dot{y}_2^2 + \dot{y}_1^2}
    \end{bmatrix}, \label{eq:hUnicycle}
\end{align}
which satisfy Assumption \ref{smp:existence} for $u_1(t) \neq 0$.
Note that $\atantwo$ is the two-argument inverse tangent with codomain $(-\pi, \pi]$.
\end{example}

Next, we formulate a constrained optimal control problem for a system governed by \eqref{eq:dynamics} under Assumptions \ref{smp:existence} and \ref{smp:bounded}.
\begin{problem} \label{prb:firstProblem}
Consider a differentially flat system \eqref{eq:dynamics} with running cost $L\big(\bm{x}(t), \bm{u}(t)\big)$ over the time horizon $[t^0, t^f]\subset\mathbb{R}$ and a final cost $\phi(\bm{x}(t^f), \bm{u}(t^f))$.
Determine the optimal control input that minimizes the total cost, i.e.,
\begin{align*}
\min_{\bm{u}(t)} ~&~ \phi\big(\bm{x}(t^f), \bm{u}(t^f)\big) + \int_{t^0}^{t^f} L\big(\bm{x}(t), \bm{u}(t)\big) dt \\
\text{subject to: }& \eqref{eq:dynamics}, \\
&\hat{\bm{g}}\big(\bm{x}(t), \bm{u}(t), t\big) \leq 0, \\
\text{given: }& \text{initial conditions}, \text{final conditions},
\end{align*}
where the initial and final states may be fixed, a function of the state variables, or left free.
In addition, the function $\hat{\bm{g}}(\bm{x}(t), \bm{u}(t), t)$ defines a vector of state and control trajectory constraints.
\end{problem}
In what follows, we present our main results, which yield a set of sufficient conditions for optimality that are only dependent on the state and control variables.

\section{Main Results} \label{sec:main}

\begin{figure*}[ht]
    \centering
    \includegraphics[width=0.9\linewidth]{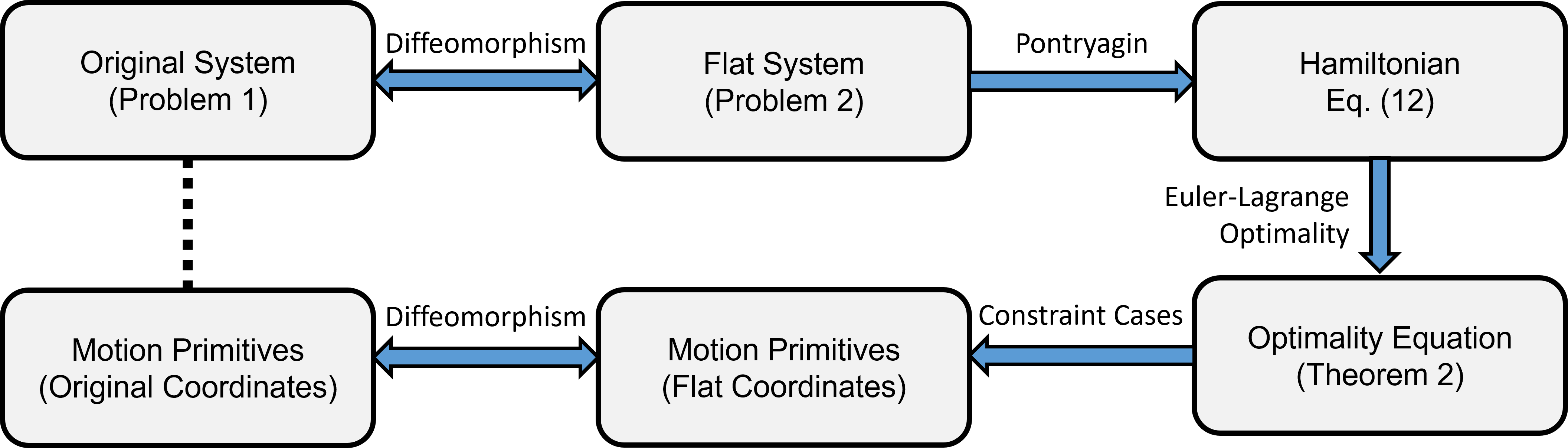}
    \caption{An overview of our proposed approach, showing how the original optimization problem is split into motion primitives in the flat space.
    These are optimally pieced together to generate the optimal trajectory in either space.}
    \label{fig:overview}
\end{figure*}

We generate the optimal solution to Problem \ref{prb:firstProblem} as follows: first, we apply the diffeomorphism of Definition \ref{def:flat} to generate an equivalent problem in the flat space.
Next, we apply Pontryagin's principle to construct the Hamiltonian in the flat space, and apply the Euler-Lagrange and optimality conditions to generate an ordinary differential equation that describes the optimal motion of the system.
We solve the differential equation to generate all possible \emph{motion primitives} that the optimality conditions admit; we achieve this by considering every possible combination of constraints that could become active along the trajectory over a non-zero time interval.
This procedure is similar to exhaustively checking every possible constraint activation in a static optimization problem to guarantee complimentary slackness as part of the KKT conditions \citep{Boyd2004ConvexOptimization}.
Finally, the resulting motion primitives can be passed back through the flatness diffeomorphism to generate the optimal motion primitives in the original coordinate system.
Thus, we generate a collection of optimal motion primitives--in both the original and flat coordinates--that must be pieced together using the optimality conditions to generate the optimal solution to Problem \ref{prb:firstProblem}.
This process is summarized in Fig. \ref{fig:overview}.

\subsection{Separability of the Optimality Conditions} \label{sec:separation}

First, we construct the flat space, which allows us to transform Problem \ref{prb:firstProblem} into an optimization over the differentially flat variables.
Note that the transformations \eqref{eq:flat2state}, \eqref{eq:flat2control}, are a function of $\bm{y} = [y_1, y_2, \dots, y_m]$ and a finite number of their derivatives.
Thus, we perform dynamic extension on each of our $i = 1, 2, \dots, m$ output variables $y_i$ by taking $k_i$ time derivatives.
The value of $k_i$ is the minimum number of derivatives required to span the domain of \eqref{eq:flat2state} and \eqref{eq:flat2control}, and thus it depends explicitly on the diffeomorphism in Definition \ref{def:flat}.
This can be achieved using the dynamic extension algorithm, as detailed in \cite{Benedetto1989RankSystems}.
Using the dynamic extension, we define analogous state and control variables for the system in the flat space.

\begin{definition} \label{def:outputSpace}
Group each output $y_i$, $i = 1, \dots, m$ and their $k_i$ derivatives into the state vector $\bm{s}(t)$ and control vector $\bm{a}(t)$ such that,
\begin{align}
    \bm{s}(t) &= \begin{bmatrix}
        y_1(t),\ \dots,\ y_1^{(k_1-1)}(t),\ \dots,\ y_m^{(k_m-1)}(t)
    \end{bmatrix}^T,\\
    \bm{a}(t) &= \begin{bmatrix}
        y_1^{(k_1)}(t),\ \dots,\ y_m^{(k_m)}(t)
    \end{bmatrix}^T,
\end{align}
and $\bm{s}\times\bm{a} \in \mathcal{Y}\times\mathcal{Y}^{(1)}\times\dots$ span the flat space.
\end{definition}

\begin{remark}
For the unicycle system in Example \ref{exa:unicycle}, the flat state and control variables are
\begin{align}
    \bm{s}(t) &= 
    \begin{bmatrix}
        y_1(t),\ y_2(t),\ \dot{y}_1(t),\ \dot{y}_2(t)
    \end{bmatrix}^T, \\
    \bm{a}(t) &= 
    \begin{bmatrix}
        \ddot{y}_1(t),\ \ddot{y}_2(t)\
    \end{bmatrix}^T,
\end{align}
which consists of two integrator chains, each with a length of $k_i = 2$, for $i = 1, 2$.
\end{remark}

With the flat space completely defined, we apply the mappings (Definition \ref{def:flat}) to construct an equivalent optimal control problem over the flat variables.

\begin{problem}\label{prb:outputProblem}
Find the cost-minimizing trajectory in the flat space,
\begin{align*}
    \min_{\bm{a}(t)} ~&~ \Phi(\bm{s}(t^f), \bm{a}(t^f)) + \int_{t^0}^{t^f} \Psi\big(\bm{s}(t), \bm{a}(t) \big)\, dt \\
    \text{subject to: }& \dot{\bm{s}} = \bm{I}(\bm{s}(t), \bm{a}(t)), \\
    & \bm{g}\big(\bm{s}(t), \bm{a}(t), t\big) \leq 0, \\
    \text{given: }& \text{initial conditions, final conditions},
\end{align*}
where $\bm{I}$ denotes integrator dynamics from Definition \ref{def:outputSpace} in Brunovsky canonical form \citep{Brunovsky1970}, while $\Phi$, $\Psi$, $\bm{g}$, and the boundary conditions are constructed by composing $\phi$, $L$, $\hat{\bm{g}}$ and the boundary conditions of Problem \ref{prb:firstProblem} with the inverse of \eqref{eq:flat2state} and \eqref{eq:flat2control}.
\end{problem}

Under the framework proposed by \cite{Bryson1975AppliedControl}, we write the constraint $\bm{g}$ with explicit dependence on the control action $\bm{a}(t)$.
This is not restrictive on our analysis, and we rigorously prove in Section \ref{sec:trajConst} that, under Assumption \ref{smp:bounded}, any trajectory constraint $\bm{h}(\bm{s}(t), t)$ can be transformed into an explicit function of the control input.
This is achieved by taking successive time derivatives of $\bm{h}(\bm{s}(t), t)$ until any component of the control vector $\bm{a}(t)$ appears; this yields a constraint with explicit functional dependence on the control variable and a set of tangency conditions that must be satisfied.
This technique is similar to the derivation of control barrier functions with high relative degree, as discussed in \cite{Xiao2019ControlDegree}.

Note that solving Problem \ref{prb:outputProblem} yields the optimal solution to Problem \ref{prb:firstProblem} through Definition \ref{def:flat}, and this construction is common in the literature (see \cite{Fliess1995FlatnessExamples,Petit2001InversionOptimization,Milam2003Real-TimeSystems,Ogunbodede2020OptimalSystems} for examples).
We present our first result next, which decouples the state and costates for the Hamiltonian function associated with Problem \ref{prb:outputProblem}.
Note that to simplify the notation, we omit the explicit dependence on $\bm{a}(t)$, $\bm{s}(t)$, and $t$ for the remainder of this Section where it does not lead to ambiguity.

We follow the standard process of \cite{Bryson1975AppliedControl,Ross2015} for solving optimal control problems.
First, we construct the Hamiltonian for Problem \ref{prb:outputProblem},
\begin{align} \label{eq:hamiltonian}
    H = \Psi(\bm{s}(t), \bm{a}(t)) &+ \bm{\lambda}^T(t) \bm{I}(\bm{s}(t), \bm{a}(t)) \notag\\
    &+ \bm{\mu}^T(t)\bm{g}\big(\bm{x}(t), \bm{a}(t), t\big),
\end{align}
where $\bm{\lambda}(t)$ is the vector of costates, $\bm{g}$ is a vector of inequality constraints, and $\bm{\mu}(t)$ is a vector of inequality Lagrange multipliers.
This leads to our first result.

\begin{theorem} \label{thm:costates}

The costates $\lambda^{y_i^{(j)}}$, for each base state $i = 1, 2, \dots, m$ and derivative $j = 0, 1, \dots, k_i - 1$, for Problem \ref{prb:outputProblem} are,
\begin{equation} \label{eq:thm1}
    \lambda^{y_i^{(j)}} = \sum_{n=1}^{k_i - j} (-1)^n \frac{d^{n-1}}{dt^{n-1}} \big( \Psi_{y_i^{(j+n)}} + \bm{\mu}^T \bm{g}_{y_i^{(j+n)}} \big),
\end{equation}
where the $\frac{d}{dt}$ operator is the Cartan field of \cite{Fliess1999ASystems}.
\end{theorem}

\begin{proof}

The Euler-Lagrange and optimality equations for \eqref{eq:hamiltonian} are,
\begin{align}
    -\dot{\bm{\lambda}}^T &= \Psi_{\bm{s}} + \bm{\lambda}^T \bm{I}_{\bm{s}} + \bm{\mu}^T\bm{g}_{\bm{s}}, \label{eq:lambdaDot} \\
    0 &=\Psi_{\bm{a}} + \bm{\lambda}^T \bm{I}_{\bm{a}} + \bm{\mu}^T\bm{g}_{\bm{a}}, \label{eq:a0} 
\end{align}
where the subscripts $\bm{a}$ and $\bm{s}$ correspond to partial derivatives with respect to those variables.
We simplify \eqref{eq:lambdaDot} by exploiting the integrator structure of $\bm{I}$ for each element of $\bm{s}(t)$. 

Note that, by construction,
\begin{align}
\bm{\lambda}^T \bm{I}_{\bm{s}} &= \big[
    0, \lambda^{y_1}, \dots, \lambda^{y_1^{(k_1-2)}}, \dots, 0, \lambda^{y_m}, \dots, \lambda^{y_m^{(k_m-2)}} \big], \\
\bm{\lambda}^T \bm{I}_{\bm{a}} &= \big[
    0, 0, \dots, \lambda^{y_1^{(k_1-1)}}, 0, 0, \dots, \lambda^{y_m^{(k_m-1)}} \big].
\end{align}
First we consider \eqref{eq:a0} for some base state $i\in\{1, 2,\dots,m\}$, which yields,
\begin{align} \label{eq:topLambda}
    0 = \Psi_{y_i^{(k_i)}} + \lambda^{y_i^{(k_i-1)}} + \bm{\mu}^T\bm{g}_{y_i^{(k_i)}},
\end{align}
which satisfies Theorem \ref{thm:costates} when $j = k_i-1$.
Next, for $j\in\{0,1,\dots,k_i-1\}$, \eqref{eq:lambdaDot} implies,
\begin{align}
    \dot{\lambda}^{y_i^{(j)}} &= -\Psi_{y_i^{(j)}} - \lambda^{y_i^{(j-1)}} - \bm{\mu}^T\bm{g}_{y_i^{(j)}}. \label{eq:lambdaDiff}
\end{align}

For the case that $j = k_i-1$, \eqref{eq:lambdaDiff} becomes,
\begin{equation} \label{eq:th1-step2}
    \dot{\lambda}^{y_i^{(k_i-1)}} = \Psi_{y_i^{(k_i-1)}} + \lambda^{y_i^{(k_i-2)}} - \bm{\mu}^T\bm{g}_{y_i^{(k_i-1)}}.
\end{equation}
Solving \eqref{eq:topLambda} for $\lambda^{y_i^{(k_i-1)}}$, taking its derivative, and substituting the result into \eqref{eq:th1-step2} satisfies Theorem \ref{thm:costates} for $j = k_i-2$.
Taking repeated time derivatives and substituting completes the proof of Theorem \ref{thm:costates}.
\end{proof}

Theorem \ref{thm:costates} could be interpreted as an alternative to the proof of separability presented in \cite{Chaplais2008InversionSystems}, however, our result is constructive and explicitly derives the costates as functions of state and control variables.
Furthermore, our result relies on differential flatness, rather than feedback linearization, and does not require affinity with respect to the control inputs in the system dynamics.
Furthermore, in the following subsections, we apply Theorem \ref{thm:costates} to generate the optimal constrained trajectory and boundary conditions as a function of the state and control variable.
This, to the best of our knowledge, has not been addressed to date.

\begin{remark}
For the unicycle system in Example \ref{exa:unicycle}, the costates are,
\begin{align}
    \bm{\lambda}^{\bm{y}} &= -\big( \psi_{\dot{\bm{y}}} + \bm{\mu}^T\bm{g}_{\dot{\bm{y}}} \big) + \frac{d}{dt} \big( \psi_{\bm{a}} + \bm{\mu}^T\bm{g}_{\bm{a}} \big), \\
    \bm{\lambda}^{\bm{y}^{(1)}} &= -\big( \psi_{\bm{a}} + \bm{\mu}^T\bm{g}_{\bm{a}} \big).
\end{align}
\end{remark}

Our next result comes from manipulating Theorem \ref{thm:costates} to eliminate the costate variables; this yields an equivalent optimality condition that is independent of the costates.

\begin{theorem} \label{thm:main}
The optimal trajectory for the system described in Problem \ref{prb:outputProblem} satisfies
\begin{equation} \label{eq:ode}
    \sum_{n=0}^{k_i} (-1)^n \frac{d^n}{dt^n} \Big( \Psi_{y_i^{(n)}} + \bm{\mu}^T\bm{g}_{y_i^{(n)}} \Big) = 0,
\end{equation}
for each integrator chain starting with the base state $y_i$, $i = 1, 2, \dots, m$.
\end{theorem}

\begin{proof}
By Theorem \ref{thm:costates}, 
\begin{equation} \label{eq:baseCostateSum}
    \lambda^{y_i} = \sum_{n=1}^{k_i} (-1)^n \frac{d^{n-1}}{dt^{n-1}}\big( \Psi_{y_i^{(n)}} + \bm{\mu}^T\bm{g}_{y_i^{(n)}}  \big),
\end{equation}
while for $j=0$ \eqref{eq:lambdaDiff} implies,
\begin{equation} \label{eq:bottomLambda}
    \dot{\lambda}^{y_i} = -\Psi_{y_i} - \bm{\mu}^T\bm{g}_{y_i}.
\end{equation}
Taking the derivative of \eqref{eq:baseCostateSum} and substituting \eqref{eq:bottomLambda} yields,
\begin{align}
    \dot{\lambda}^{y_i} &= -\Psi_{y_i} - \bm{\mu}^T\bm{g}_{y_i} \notag\\
    &= \sum_{n=1}^{k_i} (-1)^n \frac{d^{n}}{dt^n}\big( \Psi_{y_i^{(n)}} + \bm{\mu}^T\bm{g}_{y_i^{(n)}} \big),
\end{align}
which proves Theorem \ref{thm:main}.
\end{proof}

Note that while we prove Theorem \ref{thm:main} for the flat space, the mapping \eqref{eq:flat2state} and \eqref{eq:flat2control} can be composed with \eqref{eq:ode} to generate an equivalent optimality condition in the original space.
Thus, the separation of states and costates is independent of the coordinate system, and is instead a fundamental property of differentially flat systems.

\begin{remark}
Applying Theorem \ref{thm:main} to Example \ref{exa:unicycle} yields the optimality equation,
\begin{align*}
    \big(\Psi_p + \bm{\mu}^T\bm{g}_p\big) 
    - \frac{d}{dt}\big(\Psi_v + \bm{\mu}^T\bm{g}_v\big) 
    + \frac{d^2}{dt^2} \big(\Psi_a + \bm{\mu}^T\bm{g}_a\big) = 0.
\end{align*}

Furthermore, the following (arbitrary) terminal cost, running cost, and constraints,
\begin{align*}
    \phi = \frac{1}{2}u_2(t)^2, \quad
    L = \frac{1}{2}u_1(t)^2, \quad 
    \hat{g} = \theta- \theta_{\max} \leq 0,
\end{align*}
become
\begin{align*}
    \Phi &= \frac{1}{2}\Big(\frac{\ddot{y}_2\dot{y}_1 - \dot{y}_2\ddot{y}_1}{\dot{y}_2^2+\dot{y}_1^2}\Big)^2, \quad
    \Psi = \frac{1}{2} \big(\dot{y}_1(t)^2 + \dot{y}_2^2\big) , \\
    g' &= \atantwo\big(\dot{y}_2(t), \dot{y}_1(t)\big) - \theta_{\max} \leq 0.
\end{align*}
The transformed constraint $g'$ is not an explicit function of the control variables $\ddot{y}_1$ or $\ddot{y}_2$.
We resolve this by taking a single derivative of the constraint, which we call $g\coloneqq \frac{d}{dt} g'$.
The new function $g$ is an explicit function of the control variables, and we take partial derivatives of it in the optimality equation--we discuss this step in further detail in Section \ref{sec:trajConst}.
Note that we have, in essence, moved the nonlinearities of the dynamics into the objectives and constraints.
\end{remark}

While Theorem \ref{thm:main} describes the evolution of the optimal state trajectory, one must also consider instantaneous jumps in the trajectory caused by constraint activations.
Consider a constraint vector $\bm{g}$ that has $c$ linearly independent rows, then $\bm{\mu}(t)$ is a $c\times1$ matrix.
When a constraint $g_i$, $i = 1, 2, \dots, c$ does not influence the system trajectory then $\mu_i(t) = 0$ by definition, otherwise $\mu_i(t) > 0$.
When $\bm{\mu} = \bm{0}$ the trajectory is said to follow a singular (unconstrained arc), and if any $\mu_i > 0$, then the trajectory is said to follow a regular (constrained) arc.
When the system switches between singular and regular arcs, the corresponding costates may switch instantaneously at the so-called \emph{constraint junction}.

We propose a new interpretation of this property, where the collection of singular and regular arcs constitute a set of \emph{optimal motion primitives}.
A vector of $c$ constraints implies at most $2^c$ different motion primitives, which can be automatically computed using Theorem \ref{thm:main} and the corresponding constraint equations.
In other words, Theorem \ref{thm:main} provides an \emph{optimal motion primitive generator}, which can be solved  numerically or analytically to derive every possible motion primitive. 

In this context, dealing with switching elements of $\bm{\mu}(t)$ is reduced to optimally switching between a finite set of motion primitives at unknown constraint junctions.
The standard approach of \cite{Bryson1975AppliedControl} derives optimality conditions that must be satisfied at each junction,
\begin{align}
    {\bm{\lambda}^-}^T &= {\bm{\lambda}^+}^T + \bm{\pi}^T \bm{N}_{\bm{s}}, \label{eq:lambdaJump}\\
    H^+ - H^- &=  \bm{\pi}^T \bm{N}_t, \label{eq:hamJump} \\
    \frac{\partial H^-}{\partial \bm{a}^-} &= \frac{\partial H^+}{\partial \bm{a}^+} = \bm{0}, \label{eq:hamU}
\end{align}
where the superscripts $-$ and $+$ denote the instant in time just before and just after the junction, respectively, $\bm{\pi}$ is a constant vector of Lagrange multipliers, $\bm{N}$ is a vector of tangency conditions, which we rigorously derive in the following subsections, and the subscripts $\bm{s}$ and $t$ correspond to partial derivatives with respect to the state and time.
In the following subsections, we employ Theorem \ref{thm:costates} to exhaustively write the jump conditions \eqref{eq:lambdaJump}--\eqref{eq:hamU} as explicit functions of the state and control variables.
This enables us to solve Problem \ref{prb:outputProblem} using only the state and control variables, which removes the numerical instabilities that are generally associated with nonlinear optimal control.

\subsection{Interior-Point Constraints} \label{sec:interior-const}

First, we will consider the case where a set of state and/or control values are imposed at a single time instant.
Let $\bm{h}\big(\bm{s}(t_1),  t_1\big) = 0$ describe an interior point constraint that is imposed at some time $t_1$.
We construct the tangency vector,
\begin{equation} \label{eq:IPtangent}
    \bm{N}\big(\bm{s}(t), t\big) = 
    \begin{bmatrix}
        \bm{h}\big(\bm{s}(t), t\big) \\
        t - t_1
    \end{bmatrix},
\end{equation}
which is necessary and sufficient for constraint satisfaction at $t_1$ when $\bm{N}\big( \bm{s}(t_1), t_1 \big) = \bm{0}$.
Note that if the time $t_1$ is unknown, then \eqref{eq:IPtangent} reduces to $\bm{N} = \bm{h}$.
To determine the optimal jump conditions, we substitute the tangency vector \eqref{eq:IPtangent} into the optimality equations \eqref{eq:lambdaJump} and \eqref{eq:hamJump}.
Applying Theorem \ref{thm:costates} to \eqref{eq:lambdaJump}--\eqref{eq:hamU} yields $\sum_{i=1}^m \{ k_i - 1\} + 1$ equations that determine the optimal change in $\bm{a}$ and its derivatives at $t_1$, and these equations are independent of the costate vectors.

Further manipulating \eqref{eq:lambdaJump}--\eqref{eq:hamU} yields a useful pair of equations that are amenable to finding an analytical solution.
First, we substitute \eqref{eq:hamiltonian} into \eqref{eq:hamJump} and use \eqref{eq:lambdaJump} to eliminate $\bm{\lambda}^-$,
\begin{align}
    (\Psi^+ - \Psi^-) &+ ({\bm{\mu}^+}^T\bm{g}^+ - {\bm{\mu}^-}^T\bm{g}^-)  \notag\\
    &+ {\bm{\lambda}^+}^T(\bm{I}^+ - \bm{I}^-) = \bm{\pi}^T\big(\bm{N}_t + \bm{N}_s \bm{I}^-\big). \label{eq:jump2}
\end{align}
Note that, by definition, $\bm{\mu}^T\bm{g} = 0$ along the optimal state-trajectory, thus we set those terms equal to zero.
Furthermore, the state trajectory is continuous under Assumption \ref{smp:bounded} and the integrator dynamics.
Thus,
\begin{equation}
    \bm{I}^+ - \bm{I}^- 
    = \begin{bmatrix}
    \bm{0} \\
    \bm{a}^+ - \bm{a}^-
    \end{bmatrix}.
\end{equation}
Applying Theorem \ref{thm:costates} to \eqref{eq:jump2} for the case $j = k_i-1$ and simplifying yields,
\begin{align}
     (\Psi^+ - \Psi^-) ~-&~ (\Psi_{\bm{a}} + {\bm{\mu}^T}\bm{g}_{\bm{a}})^- \cdot (\bm{a}^+ - \bm{a}^-) \notag\\
     &= \bm{\pi}^T\big(\bm{N}_t + \bm{N}_s \bm{I}^+\big). \label{eq:PsiJumpInterior2}
\end{align}
Following a similar process also implies,
\begin{align}
     (\Psi^+ - \Psi^-) ~-&~ (\Psi_{\bm{a}} + {\bm{\mu}^T}\bm{g}_{\bm{a}})^+ \cdot (\bm{a}^+ - \bm{a}^-) \notag\\
     &= \bm{\pi}^T\big(\bm{N}_t + \bm{N}_s \bm{I}^-\big). \label{eq:PsiJumpInterior1}
\end{align}

\subsection{Path Constraints} \label{sec:trajConst}

Next, we consider the case when path constraints on the state and/or control variables are imposed on Problem \ref{prb:outputProblem} and influence the trajectory of the system.
To generate our optimal motion primitive using Theorem \ref{thm:main}, we first need to ensure our constraints are functions of the state and control variables.
Let $h_i\big(\bm{s}(t), t\big) \leq 0$ denote the $i = 1, 2, ..., c$ state or control constraints.
Note that $h_i$ is not required to be an explicit function of the control input.
Under the standard approach of \cite{Bryson1975AppliedControl}, we require that $h_i$ is at least $q_i-$times differentiable, where $q_i$ is the minimum number of derivatives required for any component of the control input to appear in $\frac{d^{q_i}}{dt^{q_i}} h_i$.
To guarantee satisfaction of the original constraint $h_i$, we construct the tangency vector,
\begin{equation} \label{eq:tangency}
    \bm{N}_i(\bm{s}(t), t) \coloneqq
    \begin{bmatrix}
        h_i\big(\bm{s}(t), t\big) \\
        h_i^{(1)}\big(\bm{s}(t), t\big) \\
        \vdots \\
        h_i^{(q_i-1)}\big(\bm{s}(t), t\big)
    \end{bmatrix},
\end{equation}
and define the constraint,
\begin{equation}
    g_i\big(\bm{s}(t), \bm{a}(t), t\big) \coloneqq h_i^{(q_i)}\big(\bm{s}(t), \bm{a}(t), t\big).
\end{equation}
Thus, whenever $h_i\big(\bm{s}(t), t\big) = 0$ over a non-zero interval, we impose $\bm{N}_i\big(\bm{s}(t), t\big) = 0$ and $\bm{g}_i\big(\bm{s}(t), \bm{a}(t)\big) = 0$ over the interior of the interval; this satisfies the original constraint under Assumption \ref{smp:bounded} \citep{Bryson1975AppliedControl}.
Note that, if $h_i$ is a function of the control variable, $q = 0$ and $\bm{N}_i$ is empty.
Furthermore, if the constraint is active over a zero-length interval, the problem reduces to the analysis in Section \ref{sec:interior-const} with an unknown activation time.

Finally, to construct the tangency matrix for the $c$ constraints, we construct the stacked tangency vector,
\begin{equation}
    \bm{N}\big(\bm{s}(t), t\big) = \begin{bmatrix}
        \bm{N}_{1}\big(\bm{s}(t), t\big) \\
        \bm{N}_{2}\big(\bm{s}(t), t\big) \\
        \vdots \\
        \bm{N}_{c}\big(\bm{s}(t), t\big)
    \end{bmatrix},
\end{equation}
which accounts for all of the constraints that may influence the state and control trajectory.
As with the previous section, \eqref{eq:lambdaJump}--\eqref{eq:hamU} determine the required instantaneous change in the control variables and their derivatives for an optimal trajectory.

Again, further manipulating \eqref{eq:lambdaJump}--\eqref{eq:hamU} yields a pair of useful equations.
Note that, by construction,
\begin{align}
    \bm{\pi}^T\dot{\bm{N}}^+ = 0,
\end{align}
as $\bm{N}_i = \bm{0}$ and $\bm{g}_i^+ = 0$ when constraint $i$ is active, and the corresponding $\bm{\pi}_i = 0$ otherwise.
Thus, taking the full derivative implies
\begin{align}
 \bm{\pi}^T \dot{\bm{N}}^+ = \bm{\pi}^T \Big( \bm{N}_t + \bm{N}_{\bm{s}}\cdot\bm{I}^+ \Big) = \bm{0}.
\end{align}
Thus, applying \eqref{eq:PsiJumpInterior2} at the end of a constrained motion primitive yields
\begin{equation}
    (\Psi^+ - \Psi^-) - (\Psi_{\bm{a}} + \bm{\mu}^T\bm{g}_{\bm{a}})^- \cdot (\bm{a}^+ - \bm{a}^-) = 0. \label{eq:PsiJumpArc2}
\end{equation}
This leads directly to our next result,
\begin{corollary} \label{cor:unconstrained}
If the system exits from or enters to an unconstrained motion primitive, the optimal control input satisfies
\begin{align}
    \Psi^+ - \Psi^- - \Psi_{\bm{a}}^- (a^+ - a^-) &= 0, \text{ or } \\
    \Psi^+ - \Psi^- - \Psi_{\bm{a}}^+ (a^+ - a^-) &= 0, \text{ respectively}.
\end{align}
\end{corollary}

\begin{proof}
When the system exits from an unconstrained motion primitive, $\bm{\mu}^- = \bm{0}$ and the result follows by \eqref{eq:PsiJumpArc2}.
When the system enters an unconstrained motion primitive, $\bm{\mu}^+ = \bm{0}$ and $\bm{\pi} = \bm{0}$; the result follows by \eqref{eq:PsiJumpInterior2}.
\end{proof}

\begin{corollary}\label{cor:continuity}
    If the objective function has the form $\Psi = f(\bm{s}(t)) + ||\bm{a}(t)||^2$, then the control input $\bm{a}(t)$ is always continuous when the system enters or exits an unconstrained motion primitive.
\end{corollary}

\begin{proof}
    The proof follows trivially from Corollary \ref{cor:unconstrained} and continuity in $\bm{s}(t)$ from Assumption \ref{smp:bounded}.    
\end{proof}
\subsection{Boundary Conditions} \label{sec:boundaries}

The results of Sections \ref{sec:interior-const} and \ref{sec:trajConst} completely describe the evolution of the system if the boundary conditions are known.
Next, we extend this result to the case that a boundary condition is unspecified by applying Theorem \ref{thm:costates}.

\begin{corollary} \label{cor:freeBC}
Let the state $y_i^{(j)}(t)$ for $i \in \{1, 2, \dots, m\}$ and $j \in \{ 0, 1, 2, \dots, k_i-1\}$ be unspecified at a boundary, i.e., it can be arbitrarily selected.
There exists an equivalent boundary condition that guarantees optimality of the system trajectory.
\end{corollary}

\begin{proof}
Without loss of generality, let the state variable $y_i^{(j)}(t)$ be undefined at the final time $t^f$.
Under the standard approach \cite{Bryson1975AppliedControl}, the corresponding boundary condition $\lambda^{y_i^{(j)}}(t^f) = 0$ is required to guarantee optimality.
Thus, by Theorem \ref{thm:costates},
\begin{equation}
     \sum_{n=1}^{k_i - j} (-1)^n \frac{d^{n-1}}{dt^{n-1}} \big( \Psi_{y_i^{(j+n)}} + \bm{\mu}^T \bm{g}_{y_i^{(j+n)}} \big)\Big|_{t^f} = 0
\end{equation}
is an equivalent boundary condition.
\end{proof}

In practice, it is likely that Problem \ref{prb:outputProblem} will have boundary conditions defined by functions of the state variables.
Without loss of generality, let $\bm{B}(\bm{s}(t^f), t^f) = 0$ describe the functional constraints at $t^f$.
This implies that
\begin{align}
    \bm{\lambda}^T(t^f) &= \Bigg( \frac{\partial \Phi}{\partial \bm{s}} + \bm{\nu}\frac{\partial \bm{B}}{\partial \bm{s}}  \Bigg)_{t = t^f}, \label{eq:functionBC} \\
    \bm{B}(\bm{s}(t^f), t^f) &= \bm{0},
\end{align}
where $\bm{\nu}$ is a constant Lagrange multiplier that guarantees constraint satisfaction (see: \cite{Bryson1975AppliedControl}).
Applying Theorem \ref{thm:costates} to \eqref{eq:functionBC} results in a system of equations that guarantees constraint satisfaction at the boundaries, which ensures that Problem \ref{prb:outputProblem} has the correct number of initial and final conditions.

Finally, it's possible that the boundary conditions are described at an unknown terminal time.
In this case, the optimal terminal time $t^f$ satisfies \citep{Bryson1975AppliedControl}
\begin{equation}
   \Omega = \Bigg[ \frac{\partial \Phi}{\partial t} + \bm{\nu}\frac{\partial \bm{B}}{\partial t} 
   + \Big(\frac{\partial \Phi}{\partial \bm{s}} + \bm{\nu}^T\frac{\partial \bm{B}}{\partial \bm{s}} \Big)\bm{I} + \Psi \Bigg]_{t=t^f} = 0. \label{eq:tfUnknown}
\end{equation}
Thus, Problem \ref{prb:outputProblem} always corresponds to a two-point boundary value problem with $m$ initial conditions and $m$ final conditions that are independent of the costates.
Next, we present a numerical example for generating the trajectory of a double-integrator system in real time.

\section{\blue{Robotic Manipulator Case Study}} \label{sec:case-study}

To demonstrate the effectiveness of our approach, we consider \blue{the motion planning problem for a planar serial manipulator with two revolute joints, which we refer to as `the manipulator.'
In particular, we derive the optimal trajectory for the pick-and-place problem.
}
Note that, to improve readability, we omit the explicit dependence of variables on time where it does not cause ambiguity.
\blue{
We use the standard model for our manipulator, which is depicted in Fig. \ref{fig:manipulator}.
}

\begin{figure}[ht]
    \centering
    \begin{tikzpicture}
        \node (axis) at (5,0) {};
        \node (prime) at (3,3) {};
        \draw[thick,->] (0, 0) -- (5, 0) node[midway,below] {$x$};
        \draw[thick,->] (0, 0) -- (0, 5) node[midway,left] {$y$};
        \node[fill=black,circle,inner sep=2pt,outer sep=0pt]
        (O) at (0,0) {};
        \node[fill=black,circle,inner sep=2pt,outer sep=0pt]
        (A) at (2,2) {};
        \node[fill=black,circle,inner sep=2pt,outer sep=0pt]
        (B) at (2,4) {};
        \draw[ultra thick] (O) -- (A) node[midway,above] {$l_1$};
        \draw[ultra thick] (A) -- (B) node[midway,left] {$l_2$} node[above] {$\bm{p}$};
        \draw[ultra thick,dashed] (A) -- (prime);
        \pic [draw, ->, "$\theta_1$",angle radius=1.5cm,angle eccentricity=1.2] {angle = axis--O--A};
        \pic [draw, ->, "$\theta_2$",angle radius=1.2cm,angle eccentricity=1.2] {angle = prime--A--B};
    \end{tikzpicture}
    \caption{\blue{A 2-link serial manipulator with 2 revolute joints.}}
    \label{fig:manipulator}
\end{figure}

\blue{
The state space $\bm{x} = [\theta_1, \theta_2, \dot{\theta}_1, \dot{\theta}_2]^{\intercal}$ corresponds to the joint space of the manipulator, and the action space $u = [\tau_1, \tau_2]^{\intercal}$ is the torque applied at each angle.
The manipulator's dynamics are given by,
}
\begin{equation} \label{eq:manip-dynamics}
    \blue{\bm{\tau} = D(\bm{\theta})\ddot{\bm{\theta}} + C(\bm{\theta}, \dot{\bm{\theta}})\dot{\bm{\theta}} + G(\bm{\theta})},
\end{equation}
\blue{
where $\bm{\theta} = [\theta_1, \theta_2]^{\intercal}$, $D$ is the inertial matrix, $C$ is the Coriolis matrix, and $G$ is the gravitational matrix (see \cite{spong2020robot} for further details).
}

\blue{
In this case study we consider a pick-and-place task, i.e., we seek to plan a trajectory for the grasper located at point $\bm{p}$.
The system is under-actuated; we have two control inputs, namely, the two joint torques applied to $\theta_1$ and $\theta_2$.
However, we have three states of interest: the Cartesian position of the grasper at point $\bm{p}$ and its orientation.
For pick-and-place, our variable of interest is the grasper position $\bm{p}$, and the manipulator satisfies the definition of differential flatness with $\bm{p}$ as the flat output variable.
In fact, the diffeomorphism from the joint to the state space is exactly the forward and inverse kinematics.
We also note that the inverse kinematics for the manipulator are non-unique and contain a singularity when $\theta_2 = K\pi$ for any integer $K$.
In the sequel we demonstrate that thse singularity points can be included as interior point constraints per Section \ref{sec:interior-const}--which we can either impose or avoid as part of our optimal control problem.
}

\blue{
First, we write the grasper position as an explicit function of the state variables using the forward kinematics,}
\begin{equation} \label{eq:manip-kinematics}
\blue{
    \bm{p} = 
    \begin{bmatrix}
        p_x \\ p_y
    \end{bmatrix}
    =
    l_1 \begin{bmatrix}
        \cos(\theta_1) \\ \sin(\theta_1)
    \end{bmatrix}
    + l_2 \begin{bmatrix}
        \cos(\theta_1 + \theta_2) \\ \sin(\theta_1 + \theta_2).
    \end{bmatrix}.
}
\end{equation}
\blue{
The joint angles can also be written as an explicit function of the output variables using the inverse kinematics \citep{spong2020robot}, }
\blue{
\begin{align} \label{eq:manip-inverse}
    \theta_2 &= \pm \cos^{-1}\Big(p_x^2 + p_y^2 - l_1^2 - l_2^2, \, 2\, l_1 l_2 \Big),\\
    \theta_1 &= \atantwo\Big(p_y,p_x\Big) - \atantwo\Big(l_2\sin(\theta_2), l_1 + l_2\cos(\theta_2)\Big). \notag
\end{align}
}
\blue{
Finally, composing the inverse dynamics \eqref{eq:manip-inverse} and its derivatives with the dynamics \eqref{eq:manip-dynamics} yields the control input $\bm{\tau}$ as an explicit function of the position $\bm{p}$.
Thus, the forward and inverse kinematics of the serial manipulator are exactly the diffeomorphisms of Definition \ref{def:flat}.
The resulting flat state and action space is,
}
\begin{equation}
\blue{
        \bm{s} = \begin{bmatrix}
            \bm{p} \\ \dot{\bm{p}}
        \end{bmatrix}, \quad
        \bm{a} = \ddot{\bm{p}}.
    }
\end{equation}

\blue{
Next, for the pick-and-place task, we seek to bring the manipulator from its current state at time $t=0$ and position the grasper at a desired position at some later time $T > 0$, i.e.,}
\begin{equation}
\blue{
\begin{aligned} \label{eq:manip-boundaries}
    \bm{p}(0) &= 
    l_1\begin{bmatrix}
       \cos(\theta_1) \\ \sin(\theta_1) 
    \end{bmatrix} + 
    l_2 \begin{bmatrix}
        \cos(\theta_1 + \theta_2) \\ \sin(\theta_1 + \theta_2)
    \end{bmatrix}, \\
    \dot{\bm{p}}(0) &=  \frac{d}{dt}\bm{p}(t=0), \\
    \bm{p}(T) &= \bm{p}^f, \\
    \dot{\bm{p}}(T) &= \bm{0}.
\end{aligned}
}
\end{equation}

\blue{
Note that the inverse kinematics \eqref{eq:manip-inverse} are non-unique.
Thus, any position $\bm{p}(t)$ that is non-singular at time $t$ can correspond to a `left' or `right' bend in the elbow at $\theta_2$.
We refer to these as the two `modes' of the manipulator.
The initial mode at time $t=0$ is determined by the initial state state; the final mode at time $t=T$ can be selected to influence the final orientation of the grasper.
If the initial and final modes differ, then the grasper must enter a singular configuration at some time $t_1 \in (0, T)$, i.e.,
\begin{equation} \label{eq:singular}
\begin{aligned}
    ||\bm{p}(t_1)||^2 &= (l_1 + l_2)^2, \text{ or}\\
    ||\bm{p}(t_1)||^2 &= (l_1 - l_2)^2.
\end{aligned}
\end{equation}
Thus, may we include \eqref{eq:singular} as an interior point constraint with an unknown time as per Section \ref{sec:interior-const} when the initial and final modes are distinct.
Finally, to ensure Assumption \ref{smp:existence} is satisfied, we must constrain the grasper to remain within the manipulator's workspace, i.e., }
\begin{align}
\blue{
    ||\bm{p}||^2 - (l_1 + l_2)^2 \leq 0,} \label{eq:workspace1}\\
    \blue{(l_1 - l_2)^2 - ||\bm{p}||^2 \leq 0,} \label{eq:workspace2}
\end{align}
\blue{
which coincidentally coencides with the singular configuration of this manipulator.
}

\blue{
To summarize, our approach enables us to formulate the optimal manipulator trajectory planning problem as a kinematic particle with workspace bounds \eqref{eq:workspace1}, \eqref{eq:workspace2}.
We can switch between `left' and `right' bending modes with the interior point constraint \eqref{eq:singular} if the initial and final modes are distinct, or we can constrain the manipulator to avoid singular configurations.
}

\blue{
Finally, for brevity of our analysis, we present an optimization problem that minimizies the $\mathcal{L}^2$ norm of the grasper's acceleration; this minimizes the magnitude of the force that the grasper must apply during the pick-and-place operation.
For more complex objectives, e.g., minimizing the total joint torque, the objective function must be written as an explicit function of $\bm{p}$ and any number of its derivatives using \eqref{eq:manip-inverse}.
While this may be challenging analytically, it is trivial to achieve using automatic differentiation, e.g., with Maple, Matlab, or Autodiff.
Our final optimal control problem is
}
\begin{align*}
    \min_{\ddot{\bm{a}}}  \frac{1}{2} \int_{0}^{T} \frac{1}{2} ||\bm{a}||^2 &dt \\
    \text{subject to:} & \\
    \text{integrator dynamics }& \ddot{\bm{p}} = \bm{a}, \\
    \text{initial conditions }& \eqref{eq:manip-boundaries}, \\
    \text{mode switching constraint }& \eqref{eq:singular}, \\
    \text{workspace constraints }& \eqref{eq:workspace1}, \eqref{eq:workspace2},
\end{align*}
\blue{
where the mode switching constraint is neglected if the initial and final configurations share the same mode.}

\blue{
\textbf{Optimal Motion Primitives:} 
We employ Theorem \ref{thm:main} to generate an ordinary differential equation that is sufficient for optimality,
\begin{equation} \label{eq:ex-ode}
\blue{
   \bm{\ddot{a}} + 2 \mu_i \bm{p} - 2 \mu_o \bm{p} = \bm{0}, }
\end{equation}
where $\mu_i$ and $\mu_o$ are the time-varying Lagrange multipliers corresponding to the inner and outer bounds of the workspace in \eqref{eq:workspace1} and \eqref{eq:workspace1}, respectively.
Both constraints cannot be active simultaneously, thus there are only three motion primitives:}
\begin{enumerate}
    \item \blue{Unconstrained motion, $\mu_i = \mu_o = 0$.}
    \item \blue{Inner constraint, $\mu_i \geq 0$ and $||\bm{p}|| = l_1 - l_2$.}
    \item \blue{Outer constraint, $\mu_o \geq 0$ and $||\bm{p}|| = l_1 + l_2$.}
\end{enumerate}
\blue{
The optimal trajectory is a piecewise combination of these three cases.
We construct the dynamical motion primitives from \eqref{eq:ex-ode} with the orthonormal unit vectors $\hat{\bm{p}}$ and $\hat{\bm{t}}$, which are parallel and perpandicular to the position vector $\bm{p}$, respectively.
The resulting motion primitives are, }
\begin{align}
    \ddot{\bm{a}} = \bm{0} \quad&\text{(unconstrained)} \\
\begin{aligned}
    \ddot{\bm{a}} \cdot \hat{\bm{p}} + 2(l_1 - l_2)\mu_i(t) = 0 \\
    \ddot{\bm{a}} \cdot \hat{\bm{t}} = 0
\end{aligned} \quad&(\text{inner constrained}) \\
    \begin{aligned}
        \ddot{\bm{a}} \cdot \hat{\bm{p}} + 2(l_1 + l_2)\mu_o(t) = 0 \\
        \ddot{\bm{a}} \cdot \hat{\bm{t}} = 0
    \end{aligned} \quad&(\text{outer constraint})
\end{align}
\blue{
Each dynamical motion primitive has an analytic solution,
}
\begin{align}
    \bm{a}(t) = \bm{c}_1 t + \bm{c}_2 \quad&\text{(unconstrained)}, \\
\begin{aligned}
    \bm{a}(t)\cdot\hat{\bm{p}} = \frac{\bm{v}^2}{r}\\
    \ddot{\bm{a}} \cdot \hat{\bm{t}} = 0
\end{aligned} \quad&(\text{constrained}).
\end{align}
\blue{
where $r = (l_1 + l_2)$ for the outer constraint and $r = (l_1 - l_2)$ for the inner constraint.
}

To avoid unnecessary complexity in this example, we introduce an additional assumption for this case study.
\begin{assumption} \label{smp:ex-assumption}
    The boundary conditions satisfy \blue{$(l_1 - l_2) < ||\bm{p}(t)|| < (l_1 + l_2)$, and the constraints bounding $p(t)$ are active only instantaneously.}
\end{assumption}
We only employ Assumption \ref{smp:ex-assumption} for brevity; the implication is that the optimal trajectory consists of an unknown number of unconstrained arcs connected with interior point constraints.
We have found this constraint to be reasonable for energy-minimizing systems that start and stop at rest, e.g., see \cite{Beaver2023IFAC}.

\textbf{Switching Conditions}:
Under Assumption \ref{smp:ex-assumption}, the optimal solution is a piecewise collection of \blue{unconstrained optimal motion primitives connected at junction points.
The unconstrained optimal trajectory is a system of $8$ equations and $8$ unknowns, which are the boundary conditions \eqref{eq:manip-boundaries} and $8$ unknown constants of integration for the optimal motion primitives, i.e., }
\begin{equation} \label{eq:unconstrained}
\begin{aligned}
    \bm{p} &= \bm{c}_3 t^3 + \bm{c}_2 t^2 + \bm{c}_1 t + \bm{c}_0, \\
    \bm{v} & = 3 \bm{c}_3 t^2 + 2 \bm{c}_2 t + \bm{c_1}, \\
    \bm{u} & = 6 \bm{c}_3 t + 2 \bm{c}_2
\end{aligned}
\end{equation}

In particular, the initial and final conditions are captured by a set of linear equations
\begin{align}
    A(0) \bm{c}_0 = \bm{b}_0, \\
    A(T) \bm{c}_f = \bm{b}_f,
\end{align}
where $A(0) \bm{c}_0$ and $A(T) \bm{c}_f$ denote the initial and final unconstrained trajectory segments \eqref{eq:unconstrained} evaluated at $t=0$ and $t=T$, respectively.
The vectors $\bm{c}_0$ and $\bm{c}_f$ contain the constants of integration for the initial and final unconstrained motion primitives, and $\bm{b}_0, \bm{b}_f$ are the initial and final conditions.
In the case that the unconstrained trajectory is feasible, $\bm{c}_0 = \bm{c}_f$ and the system consists of a single unconstrained arc.

\blue{
If the unconstrained trajectory is infeasible, or the initial and final modes of the manipulator are distinct, then the trajectory must transition to a singular configuration where either the inner or outer workspace constraint becomes active.
Under Assumption \ref{smp:ex-assumption}, this implies that there is only a single junction, and that it is an interior point constraint at an unknown time $t_1$.
Following Section \ref{sec:interior-const}, we first write the tangency vector with an unknown activation time,}
\begin{equation}
    \blue{N(s(t), t) = (l_1 - l_2)^2 - ||\bm{p}||^2.}
\end{equation}
\blue{
The tangency condition is satisfied by definition when $\theta_2 = \pi$; this allows us to write thie tangency condition in an equivalent form that is linear in $\bm{p}$.
We achieve this by parameterizing the point $\bm{p}$ with the unknown angle $\theta_1$,
\begin{equation}
    \bm{p}(t_1) = (l_1 - l_2) \begin{bmatrix}
        \cos(\theta_1) \\ \sin(\theta_1)
    \end{bmatrix}.
\end{equation}
}

\blue{
Next, using Theorem \ref{thm:costates} to rewrite the costates yields,}
\begin{align}
    \bm{\lambda}^v &= -\bm{a} - 2 \mu_i\bm{p} \\
    \bm{\lambda}^p &= \dot{\bm{a}} +2 \dot{\mu}_i \bm{p} - 2 \mu_i \bm{v}.
\end{align}
\blue{
Substituting these into jump in the costates \eqref{eq:lambdaJump} yields,}
\begin{align}
    \dot{\bm{a}}^+ + 2\dot{\mu}_i^+\bm{p} - 2\mu_i^+\bm{v} &= \dot{\bm{a}}^- + 2\dot{\mu}_i^-\bm{p} - 2\mu_i^-\bm{v}
    - 2\pi\bm{p}, \label{eq:ex-jump-1}\\
    -\bm{a}^+ - 2\mu_i^+\bm{p}  &= -\bm{a}^- - 2\mu_i^-\bm{p}. \label{eq:ex-jump-2}
\end{align}
\blue{
To complete our analysis take advantage of two facts,
\begin{itemize}
    \item The quantity $\bm{p}\cdot\bm{v} = 0$ in the singular configuration; this can be trivially verified using \eqref{eq:manip-kinematics}.
    \item Although $\mu_i(t)$ is problematic to evaluate at $t_1$, it is equal to zero in an open set around $t_1$; thus we take $\mu_i^- = \mu_i^+$.
\end{itemize}
}
\blue{
Thus, taking the dot product of \eqref{eq:ex-jump-1} and \eqref{eq:ex-jump-2} with $\bm{v}$ and cancelling yields,
}
\begin{align}
    \big(\dot{\bm{a}}^+ - \dot{\bm{a}}^-\big)\cdot\bm{v} &= 0, \\
    \big(\bm{a}^- - \bm{a}^+\big) &= 0.
\end{align}
\blue{
This implies continuity in the control input and the quantity $\dot{\bm{a}}\cdot\bm{v}$ at $t_1$
Thus, the optimality conditions at each junction are,}
\begin{enumerate}
    \item Continuity in the state at $t_1$: 4 equations.
    \item Tangency condition: 2 equations, 1 unknown $\theta_1$.
    \item $\bm{p}\cdot\bm{v} = 0$ at $t_1$: 1 equation.
    \item Continuity in the control input at $t_1$: 2 equations.
    \item Continuity in $\bm{a}\cdot\bm{v}$ at $t_1$: 1 equation.
\end{enumerate}

\blue{
Next, note that splitting one unconstrained arc with a junction yields $10$ unknowns ($8$ new trajectory coefficients $1$ unknown time, and the unknown parameter $\theta_1$) that we solve using the above $10$ equations.
Conditions 1, 2, and 4 are bilinear.
Thus, if we fix a time $t_1$ and angle $\theta_1$ for the junction, we can write the trajectory coefficients in the linear form,
\begin{equation} \label{eq:bilinear}
    A(t_1)\bm{c} = \bm{b}(\theta_1),
\end{equation}
where $A(t_1)$ is a square $8\times16$ matrix, $\bm{c}$ is a $16\times1$ vector containing the trajectory coefficients for both segments, and $\bm{b}(\theta_1)$ is an $8\times1$ vector that encodes the continuity and tangency conditions.
Thus, we combine \eqref{eq:bilinear} with the $8$ boundary conditions \eqref{eq:manip-boundaries} to form a block-diagonal square matrix to calculate the \textbf{optimal trajectory} for a given $t_1, \theta_1$.
Finally, we solve for the optimal values of $t_1$ and $\theta_1$ using an off-the-shelf least-squares method.
In particular, we solve}
\begin{align}
    \bm{p}(\theta_1)\cdot\bm{v}(t_1) &= 0 \label{eq:root1} \\
    \bm{a}(t_1^-)\cdot\bm{v}(t_1^-) - \bm{a}(t_1^+, \theta_1)\cdot\bm{v}(t_1^+, \theta_1) &= 0. \label{eq:root2}
\end{align}
\blue{
Note that $\bm{p}, \bm{v}, \bm{a}$ are cubic, quadratic, and linear polynomials defined by the optimal motion primitive \eqref{eq:unconstrained}.
}

\subsection{\blue{Result}}
\blue{
To demonstrate how our analytic closed-form solution to the optimal motion planning works, consider the serial manipulator of Fig. \ref{fig:manipulator} with the following parameters:
\begin{itemize}
    \item $l_1 = 3$ m, $l_2 = 2$ m
    \item $\theta_1(0) = \frac{\pi}{4}$, $\theta_2(0) = \frac{7\pi}{8}$,
    \item $\dot{\theta}_1(0) = 0, \dot{\theta}_2 = 0$
    \item $\bm{p}(T) = [-2, -3]^{\intercal}$, $\dot{\bm{p}}(T) = \bm{0}$
\end{itemize}
We also wish to have the manipulator switch modes, starting with the `left' bend configuration and ending in the `right' bend configuration.
}
\blue{
First, we calculate $\bm{p}(T)$ using \eqref{eq:manip-kinematics}.
Then, we write the boundary conditions \eqref{eq:manip-boundaries} in matrix form,
\begin{equation} \label{eq:bc-matrix}
    \begin{bmatrix}
        0 & 0 & 0 & 1 \\
        0 & 0 & 1 & 0 \\
        T^3 & T^2 & T & 1 \\
        3 T^2 & 2 T & 1 & 0
    \end{bmatrix}
    \otimes I_{2\times2}
    \begin{bmatrix}
    \bm{c}_1 \\ \bm{c}_2 \\ \bm{c}_3 \\ \bm{c}_4
    \end{bmatrix}
    = 
    \begin{bmatrix}
        \bm{p}(0) \\ \bm{0} \\ \bm{p}(T) \\ \bm{0}
    \end{bmatrix},
\end{equation}
where $\otimes$ is the Kronecker product and $I_{2\times2}$ is the $2\times2$ identity matrix.
This analytical expression for the trajectory coefficients yields the optimal unconstrained solution.
However, the resulting trajectory is infeasible as demonstrated in Fig. \ref{fig:manipulator-result}, namely, the grasper position $\bm{p}$ violates the condition $||\bm{p}|| \geq (l_1 - l_2)$.
}

\begin{figure}[h]
    \centering
    \includegraphics[width=0.8\linewidth]{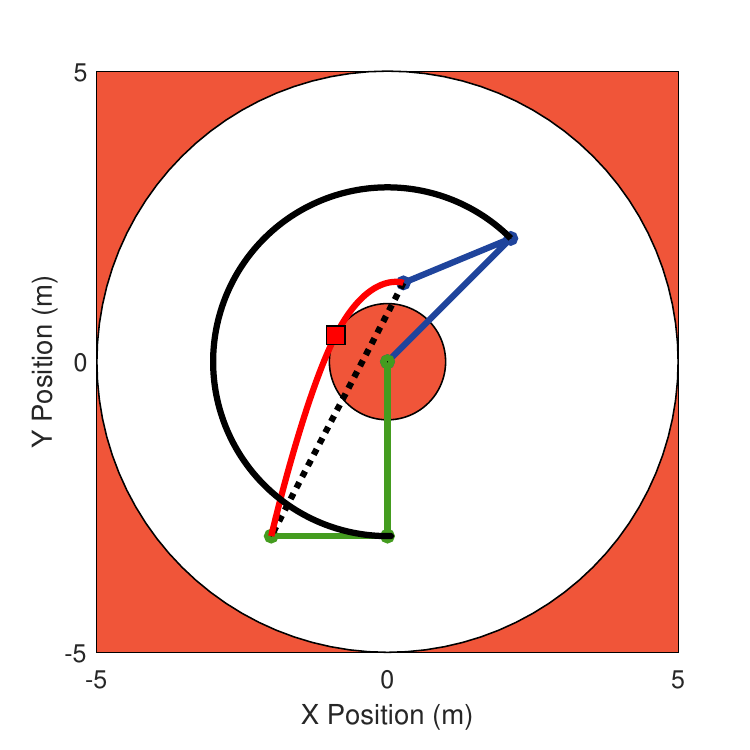}
    \caption{\blue{Initial (blue) and final (green) manipulator configuration. The unconstrained solution (dashed black), optimal solution (red line), junction (red square), and elbow trajectory (black line) are marked.}}
    \label{fig:manipulator-result}
\end{figure}

\blue{
Next, we construct the optimal trajectory from two segments, and we impose the constraint $||\bm{p}|| = (l_1 - l_2)$ as an interior constraint at some unknown time $t_1$.
If the resulting trajectory is feasible, then under Assumption \ref{smp:ex-assumption} the trajectory is also optimal.
Furthermore, this enables us to switch from the `left' to the `right' mode at the singular point.
We generate the optimal trajectory by constructing the block-diagonal matrix,
\begin{equation}
    \begin{bmatrix}
        A(0) && \bm{0} \\
        & A_C(t_1) &  \\
        \bm{0} && A(T) \\
    \end{bmatrix}
    \bm{c}
    =
    \bm{b}(\theta_1),
\end{equation}
where $A(0)$ and $A(T)$ are the boundary conditions \eqref{eq:bc-matrix}, $A_C(t_1)$ captures the bilinear continuity conditions at the unknown time $t_1$, and $\bm{0}$ is an appropriately sized zero matrix.
The vector $\bm{c}$ contains the coefficients for both trajectory segments, and $\bm{b}(\theta_1)$ encodes the continuity and tangency conditions for a given value of $\theta_1$ at the junction.
Finally, to determine the optimal time $t_1$ and angle $\theta_1$ for the junction, we solve the remaining two nonlinear equations, \eqref{eq:root1} and \eqref{eq:root2} using nonlinear least squares.
The resulting trajectory is demonstrated in Fig. \ref{fig:manipulator-result}; we note that the mean computational time required to generate the optimal trajectory is $3.5$ ms averaged over $1,000$ trials.
}

\blue{
The trajectory of the manipulator, including the joint angle trajectories, grasper acceleration, and torque applied at each joint are presented in Fig. \ref{fig:manipulator-plots}.
Note that we calculated the joint torques by taking numerical derivatives of the joint angles $\theta_1, \theta_2$ and smoothing them with a $100$ ms moving average window.
We used a mass of $0.25$ kg and a gravitational acceleration of $0$ m/s$^2$ to model a lightweight arm operating perpendicular to gravity; we computed the torque directly using \eqref{eq:manip-dynamics}.
}

\begin{figure*}[t]
    \centering
    \includegraphics[width=0.3\linewidth]{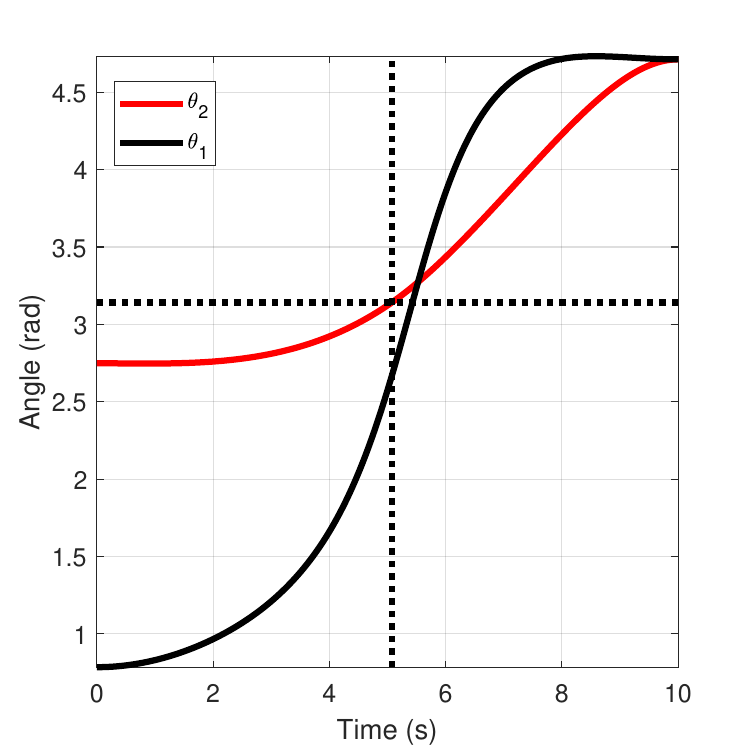}
    \includegraphics[width=0.3\linewidth]{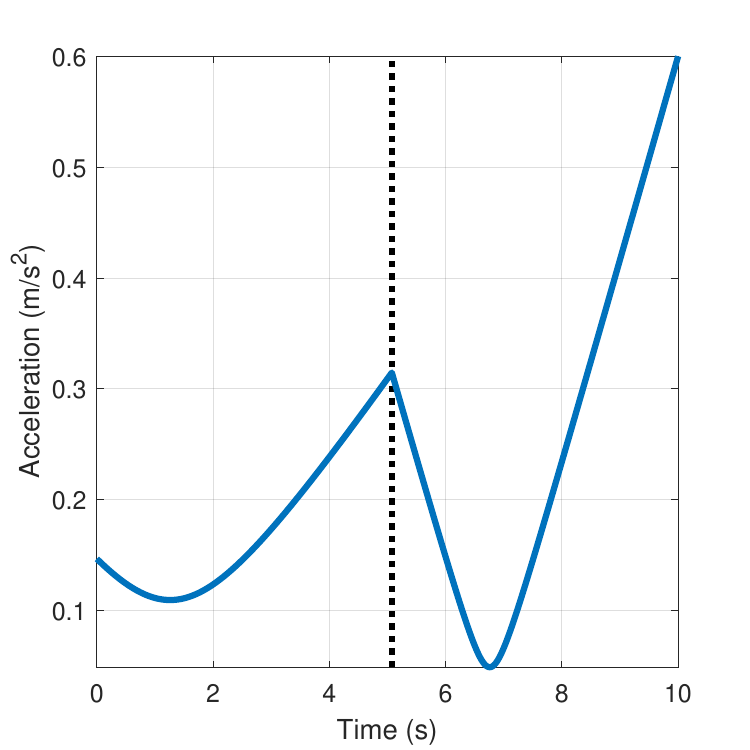}
    \includegraphics[width=0.3\linewidth]{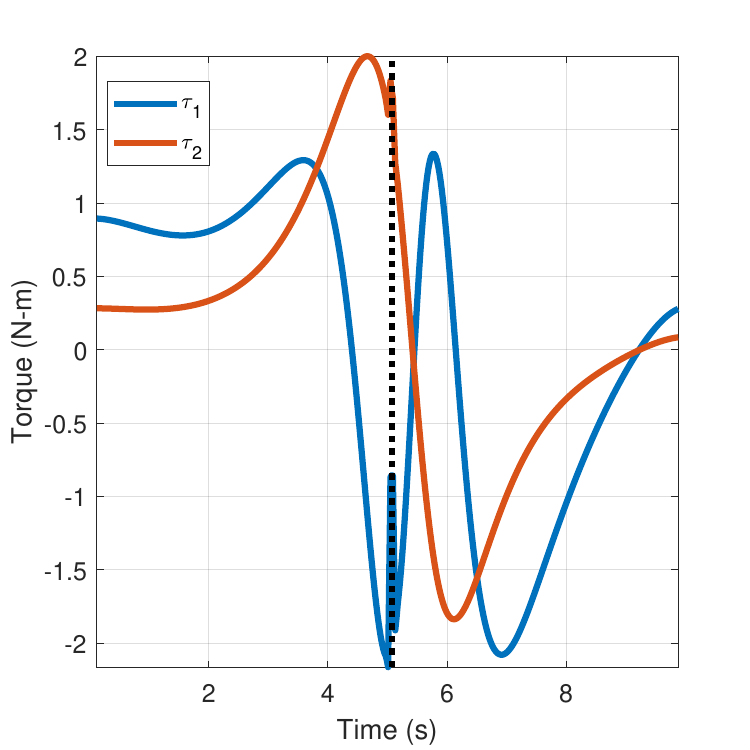}
    \caption{\blue{Plots showing the joint angles, grasper acceleration, and joint torque as a function of time. The dashed black lines denote the time that the manipulator switches from the `left' to the `right' mode at the singularity.}}
    \label{fig:manipulator-plots}
\end{figure*}

 \blue{
The smooth motion of the manipulator is clear from the joint angle and acceleration plots of Fig. \ref{fig:manipulator-plots}.
The junction occurs at approximately $t_1=5s$ with an angle of $\theta_1 = 2.7$ radians; the joint angles change gradually throughout the entire motion.
The grasper is brought toward the base of the manipulator before the junction, and it is moved away from the manipulator afterward--this leads to the corner in the acceleration magnitude that coincides with the singularity.
Finally, the torque at each joint is smooth and continuous, with only minor disturbances occurring at the singularity.
}

\section{Conclusion} \label{sec:conclusion}

In this paper, we proposed a technique to easily generate optimal trajectories for differentially flat systems.
First, we derived an explicit ordinary differential equation that describes the optimal state evolution independently of the costates. 
Second, we applied the result of Theorem \ref{thm:costates} to derive additional boundary conditions for the flat system, which has not been presented in the literature to the best of our knowledge.
Third, we proposed a motion primitive generator in Theorem \ref{thm:main} and derived the conditions to optimally switch between different motion primitives.
Finally, we applied our results in an illustrative case study, to generate \blue{smooth motion that minimizes the acceleration of a gripper for a pick-and-place operation}.
We were able to generate trajectories on the order of milliseconds, and guarantee satisfaction of the boundary conditions while \blue{respecting the worspace constraints and switching from a `left' to a `right' mode}.
Furthermore, this illustrative example is a concrete implementation of the theoretical contributions of this article.

There are several intriguing directions for future work.
First, it is practical, for given dynamics, to determine what objective functions guarantee that an analytical solution to \eqref{eq:ode} exists.
Another potential direction for future research is to relax Assumptions \ref{smp:existence} and \ref{smp:bounded} and derive similar results for systems with singularities and unbounded actuation capabilities.
Exploring problems with a large number of constraints, such as motion planning in cluttered environments, is another practical direction.
Finally, developing a general-purpose numerical method to formulate and solve optimization problems for differentially flat systems would be a valuable contribution.

\section*{Acknowledgements}

The authors would like to thank Chris Kroninger and Michael Dorothy at DEVCOM Army Research Laboratory for their insightful technical discussions.

\bibliography{mendeley, IDS_Pubs, examples, rebuttal}

\begin{thebibliography}{36}
\providecommand{\natexlab}[1]{#1}
\providecommand{\url}[1]{\texttt{#1}}
\expandafter\ifx\csname urlstyle\endcsname\relax
  \providecommand{\doi}[1]{doi: #1}\else
  \providecommand{\doi}{doi: \begingroup \urlstyle{rm}\Url}\fi

\bibitem[Beaver and Malikopoulos(2021)]{Beaver2020AnFlockingb}
L.~E. Beaver and A.~A. Malikopoulos.
\newblock {An Overview on Optimal Flocking}.
\newblock \emph{Annual Reviews in Control}, 51:\penalty0 88--99, 2021.

\bibitem[Beaver et~al.(2020)Beaver, Chalaki, Mahbub, Zhao, Zayas, and
  Malikopoulos]{Beaver2020DemonstrationCity}
L.~E. Beaver, B.~Chalaki, A.~M. Mahbub, L.~Zhao, R.~Zayas, and A.~A.
  Malikopoulos.
\newblock {Demonstration of a Time-Efficient Mobility System Using a Scaled
  Smart City}.
\newblock \emph{Vehicle System Dynamics}, 58\penalty0 (5):\penalty0 787--804,
  2020.

\bibitem[Beaver et~al.(2023)Beaver, Tron, and Cassandras]{Beaver2023IFAC}
L.~E. Beaver, R.~Tron, and C.~G. Cassandras.
\newblock A graph-based approach to generate energy-optimal robot trajectories
  in polynomial environments.
\newblock \emph{2023 IFAC World Congress (to appear)}, 2023.

\bibitem[Boyd and Vandenberghe(2004)]{Boyd2004ConvexOptimization}
S.~P. Boyd and L.~Vandenberghe.
\newblock \emph{{Convex optimization}}.
\newblock Cambridge University Press, 2004.

\bibitem[Brunovsk{\'y}(1970)]{Brunovsky1970}
P.~Brunovsk{\'y}.
\newblock A classification of linear controllable systems.
\newblock \emph{Kybernetika}, 06\penalty0 (3):\penalty0 (173)--188, 1970.

\bibitem[Bryson(1996)]{Bryson1996Optimal1985}
A.~E. Bryson, Jr.
\newblock {Optimal Control-1950 to 1985}.
\newblock \emph{IEEE Control Systems Magazine}, 16\penalty0 (3):\penalty0
  26--33, 1996.

\bibitem[Bryson and Ho(1975)]{Bryson1975AppliedControl}
A.~E. Bryson, Jr. and Y.-C. Ho.
\newblock \emph{{Applied Optimal Control: Optimization, Estimation, and
  Control}}.
\newblock John Wiley and Sons, 1975.

\bibitem[Chalaki et~al.(2022)Chalaki, Beaver, Mahbub, Bang, and
  Malikopoulos]{chalaki2021CSM}
B.~Chalaki, L.~E. Beaver, A.~M.~I. Mahbub, H.~Bang, and A.~A. Malikopoulos.
\newblock A research and educational robotic testbed for real-time control of
  emerging mobility systems: From theory to scaled experiments.
\newblock \emph{IEEE Control Systems}, 42\penalty0 (6):\penalty0 20--34, 2022.

\bibitem[Chaplais and Petit(2007)]{Chaplais2007InversionCases}
F.~Chaplais and N.~Petit.
\newblock {Inversion in indirect optimal control: constrained and unconstrained
  cases}.
\newblock In \emph{46th IEEE Conference on Decision and Control}, pages
  683--689, 2007.

\bibitem[Chaplais and Petit(2008)]{Chaplais2008InversionSystems}
F.~Chaplais and N.~Petit.
\newblock {Inversion in indirect optimal control of multivariable systems}.
\newblock \emph{ESAIM: COCV}, 14:\penalty0 294--317, 2008.

\bibitem[Di~Benedetto et~al.(1989)Di~Benedetto, Grizzle, and
  Moog]{Benedetto1989RankSystems}
M.~D. Di~Benedetto, J.~W. Grizzle, and C.~H. Moog.
\newblock Rank invariants of nonlinear systems.
\newblock \emph{SIAM Journal on Control and Optimization}, 27\penalty0
  (3):\penalty0 658--672, 1989.

\bibitem[Fliess et~al.(1995)Fliess, Levine, Martin, and
  Rouchon]{Fliess1995FlatnessExamples}
M.~Fliess, J.~Levine, P.~Martin, and P.~Rouchon.
\newblock {Flatness and defect of non-linear systems: Introductory theory and
  examples}.
\newblock \emph{International Journal of Control}, 61\penalty0 (6):\penalty0
  1327--1361, 1995.

\bibitem[Fliess et~al.(1999)Fliess, L{\'{e}}vine, Martin, and
  Rouchon]{Fliess1999ASystems}
M.~Fliess, J.~L{\'{e}}vine, P.~Martin, and P.~Rouchon.
\newblock {A lie-b{\"{a}}cklund approach to equivalence and flatness of
  nonlinear systems}.
\newblock \emph{IEEE Transactions on Automatic Control}, 44\penalty0
  (5):\penalty0 922--937, 1999.
\newblock ISSN 00189286.
\newblock \doi{10.1109/9.763209}.

\bibitem[Graichen et~al.(2010)Graichen, Kugi, Petit, and
  Chaplais]{Graichen2010HandlingExtension}
K.~Graichen, A.~Kugi, N.~Petit, and F.~Chaplais.
\newblock {Handling constraints in optimal control with saturation functions
  and system extension}.
\newblock \emph{Systems and Control Letters}, 59\penalty0 (11):\penalty0
  671--679, 11 2010.

\bibitem[Jang et~al.(2019)Jang, Vinitsky, Chalaki, Remer, Beaver, Malikopoulos,
  and Bayen]{jang2019simulation}
K.~Jang, E.~Vinitsky, B.~Chalaki, B.~Remer, L.~Beaver, A.~A. Malikopoulos, and
  A.~Bayen.
\newblock Simulation to scaled city: zero-shot policy transfer for traffic
  control via autonomous vehicles.
\newblock In \emph{Proceedings of the 10th ACM/IEEE International Conference on
  Cyber-Physical Systems}, pages 291--300, 2019.

\bibitem[Kolar et~al.(2017)Kolar, Rams, and
  Schlacher]{Kolar2017Time-optimalCrane}
B.~Kolar, H.~Rams, and K.~Schlacher.
\newblock {Time-optimal flatness based control of a gantry crane}.
\newblock \emph{Control Engineering Practice}, 60:\penalty0 18--27, 3 2017.

\bibitem[L{\'{e}}vine(2007)]{Levine2007OnLinearizability}
J.~L{\'{e}}vine.
\newblock {On The Equivalence Between Differential Flatness and Dynamic
  Feedback Linearizability}.
\newblock \emph{IFAC Proceedings Volumes}, 40\penalty0 (20):\penalty0 338--343,
  2007.

\bibitem[Li et~al.(2017)Li, Chow, Egerstedt, Lu, and
  Zhou]{Li2017MethodObstacles}
W.~Li, S.-N. Chow, M.~Egerstedt, J.~Lu, and H.~Zhou.
\newblock {Method of evolving junctions: A new approach to optimal
  path-planning in 2D environments with moving obstacles}.
\newblock \emph{The International Journal of Robotics Research}, 36\penalty0
  (4):\penalty0 403--413, 2017.

\bibitem[Mahbub and Malikopoulos(2020)]{mahbub2020sae-2}
A.~M.~I. Mahbub and A.~A. Malikopoulos.
\newblock Concurrent optimization of vehicle dynamics and powertrain operation
  using connectivity and automation.
\newblock In \emph{SAE Technical Paper 2020-01-0580}. SAE International, 2020.
\newblock \doi{10.4271/2020-01-0580}.

\bibitem[Malikopoulos et~al.(2021)Malikopoulos, Beaver, and
  Chremos]{Malikopoulos2020}
A.~A. Malikopoulos, L.~E. Beaver, and I.~V. Chremos.
\newblock Optimal time trajectory and coordination for connected and automated
  vehicles.
\newblock \emph{Automatica}, 125\penalty0 (109469), 2021.

\bibitem[Mellinger and Kumar(2011)]{Mellinger2011MinimumQuadrotors}
D.~Mellinger and V.~Kumar.
\newblock {Minimum snap trajectory generation and control for quadrotors}.
\newblock In \emph{IEEE International Conference on Robotics and Automation},
  pages 2520--2525, 2011.

\bibitem[Milam(2003)]{Milam2003Real-TimeSystems}
M.~B. Milam.
\newblock \emph{{Real-Time Optimal Trajectory Generation for Constrained
  Dynamical Systems}}.
\newblock PhD thesis, California Institute of Technology, 2003.

\bibitem[Nurkanović and Diehl(2022)]{NOSNOC}
A.~Nurkanović and M.~Diehl.
\newblock Nosnoc: A software package for numerical optimal control of nonsmooth
  systems.
\newblock \emph{IEEE Control Systems Letters}, 6:\penalty0 3110--3115, 2022.
\newblock \doi{10.1109/LCSYS.2022.3181800}.

\bibitem[Ogunbodede(2020)]{Ogunbodede2020OptimalSystems}
O.~T. Ogunbodede.
\newblock \emph{{Optimal Control of Differentially Flat Systems}}.
\newblock PhD thesis, The University at Buffalo, 2020.

\bibitem[Oh et~al.(2017)Oh, Shirazi, Sun, and Jin]{Oh2017}
H.~Oh, A.~R. Shirazi, C.~Sun, and Y.~Jin.
\newblock {Bio-inspired self-organising multi-robot pattern formation: A
  review}.
\newblock \emph{Robotics and Autonomous Systems}, 91:\penalty0 83--100, 2017.

\bibitem[Petit et~al.(2001)Petit, Milam, and
  Murray]{Petit2001InversionOptimization}
N.~Petit, M.~B. Milam, and R.~M. Murray.
\newblock {Inversion Based Constrained Trajectory Optimization}.
\newblock \emph{IFAC Proceedings Volumes}, 34\penalty0 (6):\penalty0
  1211--1216, 2001.

\bibitem[Rigatos(2015)]{Rigatos2015DifferentialControl}
G.~G. Rigatos.
\newblock {Differential flatness theory and flatness-based control}.
\newblock In \emph{Studies in Systems, Decision and Control}, volume~25, pages
  47--101. Springer International Publishing, 2015.

\bibitem[Ross(2015)]{Ross2015}
I.~M. Ross.
\newblock \emph{{A Primer on Pontryagin's Principle in Optimal Control}}.
\newblock Collegiate Publishers, San Francisco, 2nd edition, 2015.

\bibitem[Rubenstein et~al.(2012)Rubenstein, Ahler, and Nagpal]{Rubenstein2012}
M.~Rubenstein, C.~Ahler, and R.~Nagpal.
\newblock {Kilobot: A low cost scalable robot system for collective behaviors}.
\newblock In \emph{Proceedings of the 2012 IEEE International Conference on
  Robotics and Automation}, pages 3293--3298, 2012.

\bibitem[Sira-Ramirez and Agrawal(2018)]{Sira-Ramirez2018DifferentiallySystems}
H.~Sira-Ramirez and S.~K. Agrawal.
\newblock \emph{{Differentially Flat Systems}}.
\newblock 1st edition, 2018.

\bibitem[Spong et~al.(2020)Spong, Hutchinson, and Vidyasagar]{spong2020robot}
M.~W. Spong, S.~Hutchinson, and M.~Vidyasagar.
\newblock \emph{Robot modeling and control}.
\newblock John Wiley \& Sons, 2020.

\bibitem[Sreenath et~al.(2013)Sreenath, Michael, and
  Kumar]{Sreenath2013TrajectorySystem}
K.~Sreenath, N.~Michael, and V.~Kumar.
\newblock {Trajectory generation and control of a quadrotor with a
  cable-suspended load - A differentially-flat hybrid system}.
\newblock In \emph{IEEE International Conference on Robotics and Automation},
  pages 4888--4895, 2013.

\bibitem[Van~Nieuwstadt et~al.(1994)Van~Nieuwstadt, Rathinam, and
  Murray]{VanNieuwstadt1994DifferentialEquivalence}
M.~Van~Nieuwstadt, M.~Rathinam, and R.~M. Murray.
\newblock {Differential flatness and absolute equivalence}.
\newblock In \emph{Proceedings of the IEEE Conference on Decision and Control},
  volume~1, pages 326--332, 1994.

\bibitem[V{\'{a}}s{\'{a}}rhelyi et~al.(2018)V{\'{a}}s{\'{a}}rhelyi,
  Vir{\'{a}}gh, Somorjai, Nepusz, Eiben, and
  Vicsek]{Vasarhelyi2018OptimizedEnvironments}
G.~V{\'{a}}s{\'{a}}rhelyi, C.~Vir{\'{a}}gh, G.~Somorjai, T.~Nepusz, A.~E.
  Eiben, and T.~Vicsek.
\newblock {Optimized flocking of autonomous drones in confined environments}.
\newblock \emph{Science Robotics}, 3\penalty0 (20), 2018.

\bibitem[Xiao and Belta(2019)]{Xiao2019ControlDegree}
W.~Xiao and C.~Belta.
\newblock {Control Barrier Functions for Systems with High Relative Degree}.
\newblock In \emph{Proceedings of the IEEE Conference on Decision and Control},
  volume 2019-December, pages 474--479. Institute of Electrical and Electronics
  Engineers Inc., 12 2019.

\bibitem[Zhai et~al.(2022)Zhai, Hou, Zhang, and Zhou]{Zhai2022MethodFields}
H.~Zhai, M.~Hou, F.~Zhang, and H.~Zhou.
\newblock {Method of evolving junction on optimal path planning in flows
  fields}.
\newblock \emph{Autonomous Robots}, 46\penalty0 (8):\penalty0 929--947, 12
  2022.

\end{thebibliography}


\begin{wrapfigure}{l}{0.10\textwidth}
\centering
\includegraphics[width=0.13\textwidth]{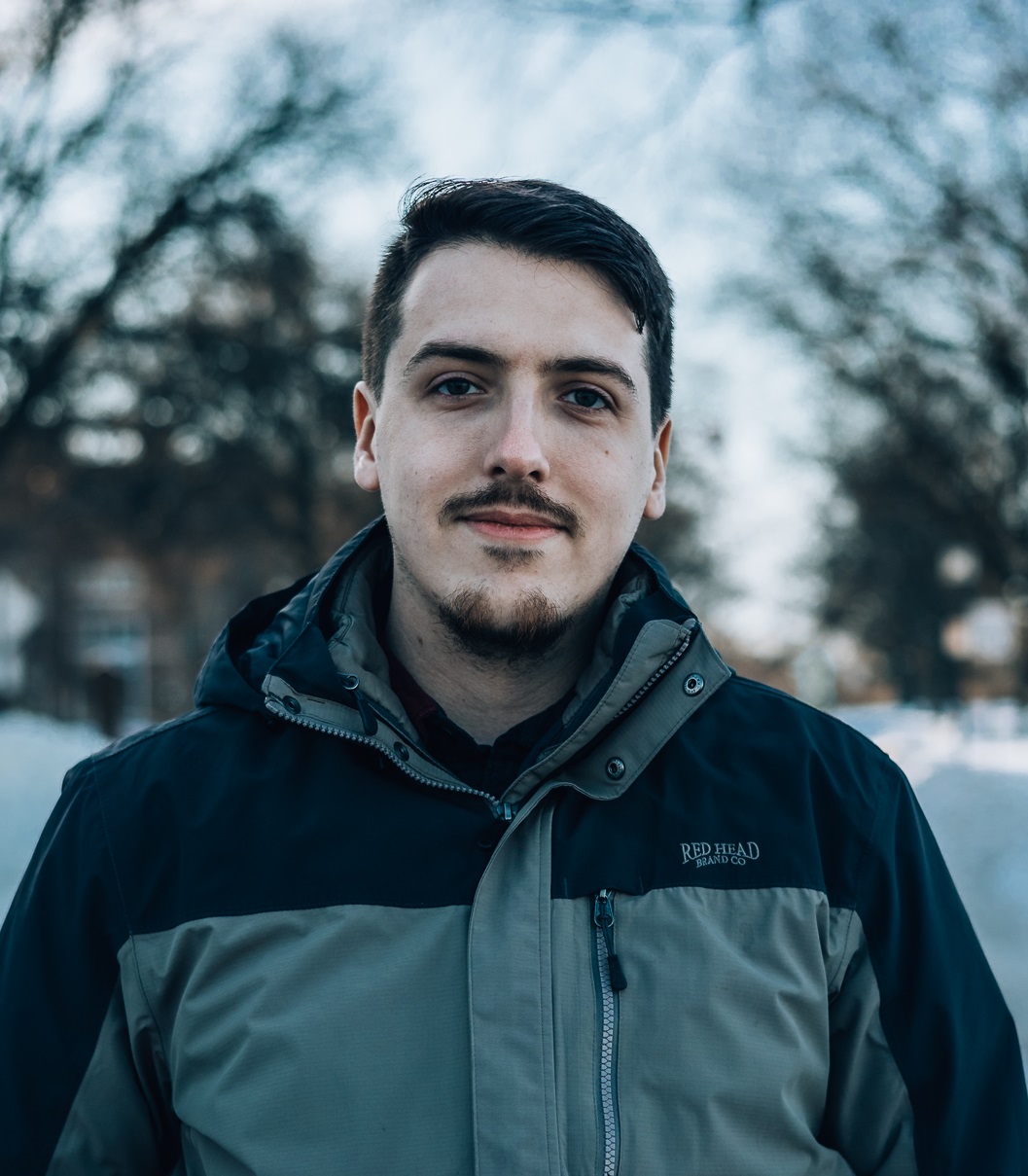}
\end{wrapfigure}
\textbf{Logan Beaver} received the B.S. degree in mechanical engineering from the Milwaukee School of Engineering, Milwaukee, WI, USA, in 2015, the M.S. degree in mechanical engineering from Marquette University, Milwaukee, WI, USA, in 2017, and the Ph.D. degree in mechanical engineering from the University of Delaware, Newrk, DE, USA in 2022.
He was a postdoc with the Division of Systems Engineering at Boston University from 2022--2023, and is currently an assistant professor of autonomous systems in the Department of Mechanical and Aerospace Engineering at Old Dominion University, Nrofolk, VA, USA.
His research interests are at the interface of complex systems, decentralized control, and optimization; his focus is on engineering decentralized robotic systems that take advantage of the mechanism of emergence. He is a member of the IEEE, SIAM ASME, and AAAS.

\par
\begin{wrapfigure}{l}{0.10\textwidth}
\centering
\includegraphics[width=0.13\textwidth]{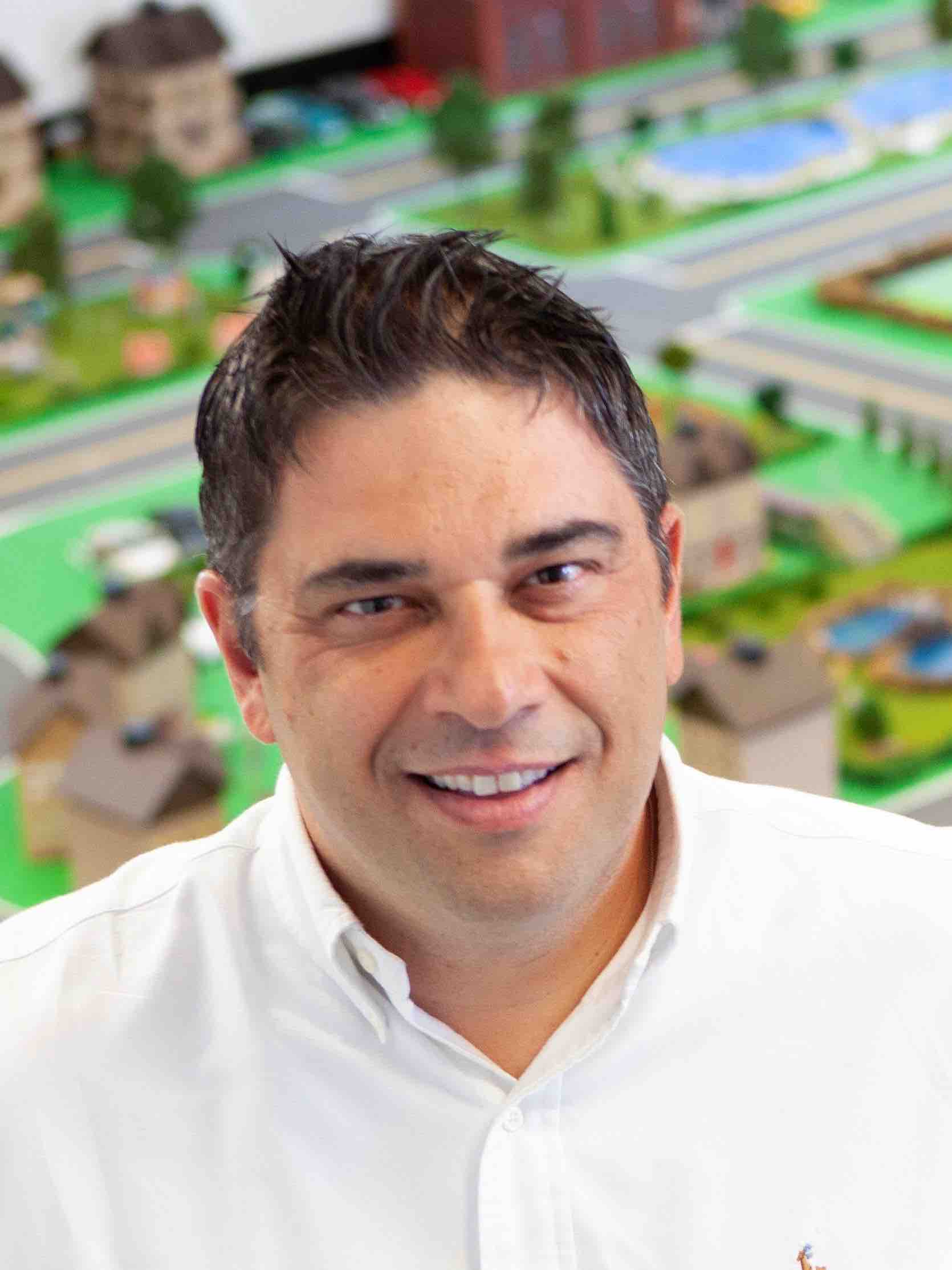}
\end{wrapfigure}
\textbf{Andreas A. Malikopoulos} received the Diploma in mechanical engineering from the National Technical University of Athens, Greece, in 2000. He received M.S. and Ph.D. degrees from the department of mechanical engineering at the University of Michigan, Ann Arbor, Michigan, USA, in 2004 and 2008, respectively. He is a Professor in the School of Civil and Environmental Engineering at Cornell University and the Director of the Information and Decision Science (IDS) Laboratory.
Prior to these appointments, he was the Director of the Sociotechnical Systems Center at the University of Delaware, the Deputy Director and the Lead of the Sustainable Mobility Theme of the Urban Dynamics Institute at Oak Ridge National Laboratory, and a Senior Researcher with General Motors Global Research \& Development. His research spans several fields, including analysis, optimization, and control of cyber-physical systems; decentralized systems; stochastic scheduling and resource allocation problems; and learning in complex systems. The emphasis is on applications related to smart cities, emerging mobility systems, and sociotechnical systems. He is an Associate Editor of Automatica and IEEE Transactions on Automatic Control. He is a member of SIAM, AAAS. He is also a Senior Member of IEEE, and a Fellow of the ASME.

\end{document}